\documentclass[10pt,twocolumn,letterpaper]{article}

\usepackage{cvpr}
\usepackage{times}
\usepackage{epsfig}
\usepackage{graphicx}
\usepackage{amsmath}
\usepackage{amssymb}

\usepackage[font={footnotesize},up,bf]{caption}

\usepackage{subfig}

\usepackage[table]{xcolor}
\usepackage{multirow}

\usepackage{arydshln}
\usepackage{xspace}

\usepackage[linesnumbered, algoruled, vlined]{algorithm2e}
\usepackage{eucal,bibspacing,cite,verbatim}

\usepackage{amsthm}
\usepackage{proof}

\def\fasth{{FastHash}\xspace}

\newtheorem{proposition}{Proposition}

\def\best{\bf \cellcolor[gray]{0.9}}

\def\sign{\operatorname*{sign\,}}
\def\ssum{\textstyle\sum}

\def\x{{\boldsymbol x}}
\def\w{{\boldsymbol w}}

\def\z{{\boldsymbol z}}

\def\bh{{\Phi}}

\def\bT{{\bf  T}}
\def\tree{{T}}

\def\Real{\mathbb{R}}

\def\Z{{\bf  Z}}
\def\Y{{\bf  Y}}

\def\bY{{\bf  Y}}

\def\block{{\mathcal{B}}}

\DeclareMathOperator{\Ucal}{\mathcal{U}}
\DeclareMathOperator{\Vcal}{\mathcal{V}}

\DeclareMathOperator{\Xcal}{\mathcal{X}}

\expandafter\def\expandafter\normalsize\expandafter{%
\normalsize\setlength\abovedisplayskip{3pt}}

\expandafter\def\expandafter\normalsize\expandafter{%
\normalsize\setlength\belowdisplayskip{3pt}}

\usepackage[pagebackref=true,breaklinks=true,letterpaper=true,colorlinks,bookmarks=false]{hyperref}

\cvprfinalcopy %

\ifcvprfinal\fi
\begin{document}

\title{Fast Supervised Hashing with Decision Trees for High-Dimensional Data}

\author{Guosheng Lin,
~ Chunhua Shen\thanks{Corresponding should be addressed to C. Shen.
    Code is available at:
   \url{https://bitbucket.org/chhshen/fasthash/}.},
~ Qinfeng Shi,
~ Anton van den Hengel,
~ David Suter\\
The University of Adelaide, SA 5005, Australia\\
}

\maketitle

\thispagestyle{empty}

\begin{abstract}

Supervised hashing aims to map the original features to compact binary codes that are
able to preserve label based similarity in the Hamming space.
Non-linear hash functions have demonstrated their advantage over
linear ones due to their powerful generalization capability.
In the literature, kernel functions are typically used to achieve non-linearity in hashing,
which achieve encouraging retrieval performance at the price of slow evaluation and training time.
Here we propose to use boosted decision trees for achieving non-linearity in hashing,
which are fast to train and evaluate, hence more suitable for hashing with high dimensional data.
In our approach, we first propose sub-modular formulations for the hashing binary code
inference problem and an efficient GraphCut based block search method for solving
large-scale inference.
Then we learn hash functions by training boosted decision trees to fit the binary codes.
Experiments demonstrate that our proposed method significantly outperforms most state-of-the-art methods
in retrieval precision and training time.
Especially for high-dimensional data, our method is orders of magnitude faster than many
methods in terms of training time.

\end{abstract}

\section{Introduction}

Hashing methods construct a set of hash
functions that map the original features into
compact binary codes.
Hashing enables fast search by using look-up tables
or hamming distance based ranking.
Moreover, compact binary
codes are extremely efficient for large-scale data storage.
Applications include image retrieval,
large-scale object detection \cite{fastdection} and so on.

Hashing methods aim to preserve some notion of
similarity (or distance) in the Hamming space.
These methods
can be roughly categorized as supervised and unsupervised.
Unsupervised hashing methods
try to preserve the similarity in the original feature space. For example, Locality-Sensitive Hashing (LSH)
\cite{Gionis1999} randomly generates linear hash functions to
approximate cosine similarity;
Spectral Hashing \cite{MDSH} learns eigenfunctions that preserve Gaussian affinity;
Iterative Quantization (ITQ) \cite{gong2012iterative} approximates the Euclidean distance in the Hamming space; and
Hashing on manifolds \cite{CVPR13aShen} takes the intrinsic manifold structure into consideration.

Supervised hashing is designed to preserve some label-based similarity \cite{KSH,TSH,kulis2009learning,cgh}.
This might take place, for example, in the case where
images from the same category are defined as being semantically similar to each other.
Supervised hashing has received increasingly attention recently
such as Supervised Hashing with Kernels (KSH) \cite{KSH}, Two-Step Hashing (TSH) \cite{TSH},
Binary Reconstructive embeddings (BRE) \cite{kulis2009learning}.
Although supervised hashing is more flexible and appealing for real-world applications, the learning is usually
much slower than that of unsupervised hashing.
Despite the fact that hashing is only of practical interest in the case where it may be applied to large numbers of high-dimensional features, most supervised hashing approaches are demonstrated only on relatively small numbers of low dimensional features.
For example, codebook based features have achieved remarkable success on image classification \cite{coates2011importance, kiros12}, of which the number of feature dimension usually comes to tens of thousands.
To exploit this recent advance of feature learning,
it is very desirable for supervised hashing to be able to deal with
large-scale data efficiently on sophisticated high-dimensional features.
To bridge this gap, we propose a supervised hashing method which is able to leverage large training sets
and efficiently incorporate with high-dimensional features.

Non-linear hash functions, e.g., the kernel hash function employed in KSH and TSH, have shown much improved performance over the linear hash function. %
However, kernel functions could be extremely expensive for both training and testing on high-dimensional features.
Thus a scalable supervised hashing method with non-linear hash functions is desirable too.

Our main { contributions} are as follows.
{(i)} We propose to use (ensembles of) decision trees as hash functions for supervised hashing,
    which can easily deal with a very large number of training data with high dimensionality (tens of thousands),
    and has the desirable non-linear mapping. To  our knowledge, our method is the first general
    hashing method that uses decision trees as hash functions.
{(ii)} In order to efficiently learn decision trees for supervised hashing,
    we apply a two-step learning strategy which decomposes the learning into the binary code inference and the simple binary classification training of decision trees.
For binary code inference, we propose sub-modular formulations and an efficient GraphCut \cite{boykov2001fast} based block search method for solving large-scale inference.
{(iii)} Our method significantly outperforms many state-of-the-art methods in terms of retrieval precision.
For high-dimensional data, our method is usually orders of magnitude faster
in terms of training time.

The two-step learning strategy employed in our method is inspired by the recent work of TSH \cite{TSH}.
Other work
in \cite{shakhnarovich2003fast,torralba2008small,rastegari2012attribute} also learns hash functions by training classifiers.
The spectral method in TSH for binary code inference does not scale
well on large training data,
     and it may also lead to inferior result due to the loose relaxation of spectral methods.
     Moreover, TSH only demonstrates satisfactory performance with kernel hash functions
     on small-scale training data with low dimensionality,
     which is clearly not practical for large-scale learning on high-dimensional features.
    In contrast with TSH, we explore efficient decision trees as hash functions and
    propose an efficient GraphCut based method for binary code inference.
    Experiments show that our method significantly outperforms TSH.

\begin{algorithm}[t!]

	\SetAlFnt{\footnotesize}

	\footnotesize{

	\caption{\footnotesize An example for constructing blocks}

	\label{alg:block}
	    \KwIn{Training data points: $\{\x_1,...\x_n\}$; Affinity matrix: $\Y$.}
		\KwOut{blocks:$\{ \block_1, \block_2, ...\}$}     $\Vcal \leftarrow \{\x_1,...,\x_n\}$; $t=0$\;
		\Repeat{$\Vcal = \emptyset$}     {
			$t=t+1$; $\block_t \leftarrow \emptyset$; $\x_i$: randomly selected from	$\Vcal$\;
			initialize $\Ucal$ as joint of $\Vcal$ and similar examples of $\x_i$ \;
			\For{each $\x_j$ in $\Ucal$}{
				\If{$\x_j$ is not dissimilar with any examples in $\block_t$}{ add $\x_j$ to $\block_t$;
				remove $\x_j$ from $\Vcal$ \;}
			}
		}
	}

\end{algorithm}

\section{The proposed method}

Let $\Xcal=\{\x_1, ..., \x_n\} \subset \Real^d$ denote a set of training points.
Label based similarity information is described by an affinity matrix: $\bY$, which is the ground truth for supervised learning.
The element in $\bY$: $y_{ij}$ indicates the similarity of two data point $\x_i$ and $\x_j$; and $y_{ij}=y_{ji}$.
Specifically, $y_{ij}=1$ if two data points are similar, $y_{ij}=-1$ if dissimilar (irrelevant)
and $y_{ij}=0$ if the pairwise relation is undefined.
We aim to learn a set of hash functions to preserve the label based similarity in the Hamming space.
$m$ hash functions are denoted as: $\bh(\x)=[h_1(\x),
..., h_m( \x  )]$.
	The output of hash functions are $m$-bit binary codes: $\bh(\x) \in \{-1,1\}^m$.
Closely related to Hamming distance, the Hamming affinity is calculated by the inner product of two binary codes : $s(\x_i, \x_j)=\sum_{k=1}^mh_k(\x_i)h_k(\x_j)$. Similar to KSH \cite{KSH}, we formulate hashing learning based on Hamming affinity, which is to encourage positive affinity value of similar data pairs and negative for dissimilar data pairs. The optimization is written as:
 \begin{align}
 	\label{eq:opt_main}
	\min_{\bh(\cdot)} \sum_{i=1}^n\sum_{j=1}^n
     |y_{ij}| \biggr[m y_{ij} - \sum_{k=1}^m h_k(\x_i)h_k(\x_j) \biggr]^2.
 \end{align}
Note that KSH does not include the multiplication of $|y_{ij}|$ in the objective.
We use $|y_{ij}|$ to prevent undefined pairwise relation from harming the hashing task.
If the relation is undefined, $|y_{ij}|=0$, otherwise, $|y_{ij}|=1$.
In contrast to KSH which uses kernel functions, here we employ decision trees as hash functions.
We define each hash function as a linear combination of decision trees, that is,
\begin{align}
 	\label{eq:btree}
	h(\x)=\sign(\ssum_{q=1}^Q  w_q \tree_q(\x)).
 \end{align}
Here $Q$ is the number of decision trees.
$\tree(\cdot) \in \{-1, 1\}$ denotes a tree function with binary output;
The weighting $\w =[ w_1, ..., w_Q]$ and trees $\bT = [\tree_1, ..., \tree_Q]$ are parameters we need to learn for one hash function. Comparing to kernel method, decision trees enjoy faster testing on high-dimensional data as well as the non-linear fitting ability.

Optimizing \eqref{eq:opt_main} directly for learning decision trees is difficult, and the
technique used in KSH is no longer applicable.
Inspired by TSH \cite{TSH}, we introduce auxiliary variables $z_{k,i} \in \{-1, 1\}$
as the output of the $k$-th hash function on $\x_i$: $z_{k,i}=h_k(\x_i)$.
Clearly,
$z_{k,i}$ is the binary code of $i$-th data point in the $k$-th bit.
With these auxiliary variables, the problem \eqref{eq:opt_main} can be decomposed into two
sub-problems:
\begin{subequations}
\begin{align}
	\min_{\Z \in \{-1, 1\}^{ m \times n}} & \sum_{i=1}^n\sum_{j=1}^n
     |y_{ij}| \biggr(m y_{ij} - \sum_{k=1}^m z_{k, i}z_{k, j} \biggr)^2; \label{eq:opt_step1}\\
    \min_{\bh(\cdot)} \;\; & \ssum_{k=1}^{m} \ssum_{i=1}^n \delta( z_{k,j}=h_k(\x_i) ). \label{eq:opt_step2_tmp}
\end{align}
\end{subequations}
Here $\Z$ is the matrix of $m$-bit binary codes for all training data points.
Note that \eqref{eq:opt_step1} is a binary code inference problem, and \eqref{eq:opt_step2_tmp} is a simple binary classification problem. This way, the complicated decision trees learning for supervised hashing \eqref{eq:opt_main} now becomes two relatively simpler tasks---solving \eqref{eq:opt_step1} (Step 1)
and \eqref{eq:opt_step2_tmp}  (Step 2).

{\bf Step 1: Binary code inference.}
For \eqref{eq:opt_step1}, we sequentially optimize for one bit at a time, conditioning on previous bits.
When solving for the $k$-th bit, the cost in \eqref{eq:opt_step1} is written as:
\begin{align}
& \ssum_{i=1}^n\ssum_{j=1}^n
     |y_{ij}| (k y_{ij} - \ssum_{p=1}^k z_{p, i}z_{p, j} )^2 \notag \\
     = & \ssum_{i=1}^n\ssum_{j=1}^n
     |y_{ij}| (k y_{ij} - \ssum_{p=1}^{k-1} z_{p, i}z_{p, j} - z_{k, i}z_{k, j} )^2 \notag \\
     = & \ssum_{i=1}^n\ssum_{j=1}^n 
     -2 |y_{ij}| ( k y_{ij} - \ssum_{p=1}^{k-1} z_{p, i}z_{p, j} ) z_{k, i}z_{k, j} \notag \\
      & + const. 
\end{align}
Hence the optimization for the $k$-th bit can be equivalently formulated as a binary quadratic problem:
\begin{subequations}
\begin{align}
 	\label{eq:opt_step1_bqp}
	 \min_{\z_k \in \{-1, 1\}^{n}} & \ssum_{i=1}^n\sum_{j=1}^n
      a_{ij} z_{k, i}z_{k, j}, \\
      \text{where,} \; & a_{ij}  = -|y_{ij}| ( k y_{ij} - 
      \ssum_{p=1}^{k-1} z^\ast_{p, i} z^\ast_{p, j} ). \label{eq:opt_step1_bqp2}
\end{align}
\end{subequations}
Here $z^\ast$ denotes a binary code in previous bits.
We use a stage-wise scheme for solving each bit.
Specifically, when solving for the $k$-th bit, the bit length is set to $k$ instead of $m$, 
which is shown in \eqref{eq:opt_step1_bqp2}.
In this way, the optimization of current bit depends on the loss 
caused by previous bits, 
which usually leads to better inference results.

\begin{algorithm} [t!]
	\SetAlFnt{\footnotesize}
	\footnotesize{

	\caption{\footnotesize
				Step 1: Block GraphCut for binary code inference
			}
		\label{alg:step1}
		\KwIn{Affinity matrix: $\Y$; bit length: $k$; max
			  inference iteration; blocks:$\{ \block_1, \block_2, ...\}$;
			   binary codes: $\{\z_1, ..., \z_{k-1} \}$.
		}
		\KwOut{Binary codes of one bit: $\z_k$}
		\Repeat{max iteration is reached}
		{
			Randomly permute all blocks\;
			\For{each $\block_i$}
			{
				Solve the inference in \eqref{eq:opt_step1_block} on $\block_i$ using GraphCut\;
			}
		}
	}
\end{algorithm}

Alternatively, one can apply spectral relaxation method to solve \eqref{eq:opt_step1_bqp}, as in TSH.
However solving eigenvalue problems does not scale up to large training sets, and the spectral relaxation is rather loose (hence leading to inferior results).
Here we propose sub-modular formulations for the binary code inference problem and an efficient GraphCut based block search method for solving large-scale inference.
We first group data points into a number of blocks, then optimize the corresponding variables of one block at a time while
conditioning on
the rest of the variables.
Let $\block$ denote a block of data points.
The cost in \eqref{eq:opt_step1_bqp} can be rewritten as:
\begin{subequations}
\begin{align}
	& \ssum_{i=1}^n\ssum_{j=1}^n
      a_{ij} z_{k, i}z_{k, j} \notag \\
     = & \ssum_{i \in \block} \ssum_{j \in \block} a_{ij} z_{k, i}z_{k, j} 
      	 +  \ssum_{i \in \block} \ssum_{j \notin \block} a_{ij} z_{k, i} z_{k, j} \notag \\
      	 & +  \ssum_{i \notin \block} \ssum_{j \in \block} a_{ij}  z_{k, i}  z_{k, j} 
      	 +  \ssum_{i \notin \block} \ssum_{j \notin \block} a_{ij}   z_{k, i}   z_{k, j}. \notag
\end{align}
\end{subequations}
When optimizing for one block, those variables which are not involved in the target block are 
set to constants.
Hence, the optimization for one block can be written as:
 \begin{align}
 	\label{eq:opt_step1_block_tmp}
	\min_{\z_{k, \block} \in \{-1, 1\}^{|\block|}}
	& \ssum_{i \in \block}  \ssum_{j \in \block} a_{ij} z_{k, i} z_{k, j}  \notag \\
      	 & +  2 \ssum_{i \in \block} \ssum_{j \notin \block} a_{ij} z_{k, i} \hat z_{k, j}.
\end{align}
Here $\hat z_k$ denotes a binary code of the $k$-th bit which is not involved in the target block.
With the definition of $a_{ij}$ in \eqref{eq:opt_step1_bqp2}, 
the optimization for one block can be written as:
\begin{subequations}
\label{eq:opt_step1_block_all}
\begin{align}
 	\label{eq:opt_step1_block}
	\min_{\z_{k, \block} \in \{-1, 1\}^{|\block|}}  
	\sum_{i \in \block} u_i z_{k,i} +  \sum_{i \in \block}\sum_{j \in \block}
      v_{ij} z_{k,i} z_{k,j},
\end{align}
\begin{align}
       \text{where,} & \;   v_{ij}  = -|y_{ij}| ( k y_{ij} 
       - \ssum_{p=1}^{k-1} z^\ast_{p, i} z^\ast_{p, j} ), \label{eq:opt_step1_block2} \\
        u_i=  & - 2 \ssum_{j \notin \block} \hat  z_{k,j} |y_{ij}| (k y_{ij}  
        - \ssum_{p=1}^{k-1} z^\ast_{p, i} z^\ast_{p, j} ).
         \label{eq:opt_step1_block3}
\end{align}
\end{subequations}
Here $u_i, v_{ij}$ are constants.
The key to construct a block is to ensure \eqref{eq:opt_step1_block} of
such a block is sub-modular, so we can apply efficient GraphCut.
We refer to this as Block GraphCut (Block-GC), shown in Algorithm \ref{alg:step1}.
Specifically in our hashing problem, by leveraging similarity information, we can easily
construct blocks which meet the sub-modular requirement, as shown in the following
proposition:
\begin{proposition}
\label{pro:p1}
	$\forall i,j \in \block$, if $y_{ij} \geq 0 $,
	the optimization in \eqref{eq:opt_step1_block} is a sub-modular problem.
In other words, for any data point in the block, if it is {\textit not} dissimilar with any other data points in the block,
   then \eqref{eq:opt_step1_block} is sub-modular.
\end{proposition}
\begin{proof}
If  $y_{ij} \geq 0$, $ k y_{ij} \geq \ssum_{p=1}^{k-1} z^\ast_{p, i} z^\ast_{p, j}  $ holds. Thus $v_{ij}  = -|y_{ij}| ( k y_{ij} - \ssum_{p=1}^{k-1} z^\ast_{p, i} z^\ast_{p, j} ) \leq 0. $
Let $\theta_{ij}(z_{k,i},z_{k,j})=v_{ij} z_{k,i} z_{k,j}$, we have $\theta_{ij}(-1,1)= \theta_{ij}(1,-1)=-v_{ij} \geq 0; \theta_{ij}(1,1)= \theta_{ij}(-1,-1)= v_{ij} \leq 0$. Hence $\forall i,j \in \block, \theta_{ij}(1,1) + \theta_{ij}(-1,-1) \leq 0 \leq \theta_{ij}(1,-1) + \theta_{ij}(-1,1)$, which prove the sub-modularity of \eqref{eq:opt_step1_block} \cite{rother2007optimizing}.
\end{proof}
Blocks can be constructed in many ways as long as they satisfy the condition in Proposition \ref{pro:p1}.
A simple greedy method is shown in Algorithm \ref{alg:block}.
Note that the blocks can overlap and the union of them needs to cover all $n$ variables.
If one variable is one block, Block-GC becomes ICM \cite{besag1986statistical, UGM} which optimizes for one variable at a time.

\begin{algorithm}[t!]

	\SetAlFnt{\footnotesize}
	\footnotesize{

	\caption{\footnotesize FastHash}

	\label{alg:main}

	\KwIn{Training data points: $\{\x_1,...\x_n\}$;
	Affinity matrix: $\Y$; bit length: $m$; blocks:$\{ \block_1, \block_2, ...\}$. }
	\KwOut{Hash functions: $\bh=[h_1, ..., h_m]$}
	\For{ $k=1,...,m$}
	{
		Step-1: call Algorithm \ref{alg:step1} to obtain binary codes of $k$-th bit\;
		Step-2: train trees in \eqref{eq:opt_step2} to obtain hash function $h_k$\;
		update the binary codes of $k$-th bit by the output of $h_k$\;
	}

	}

\end{algorithm}

{\bf Step 2: Learning boosted trees as hash functions.}
For binary classification in \eqref{eq:opt_step2_tmp}, usually the zero-one loss is replaced by some convex surrogate loss.
Here we use the exponential loss which is common for boosting methods.
The classification problem for learning the $k$-th hash function is written as:
\begin{align}
 	\label{eq:opt_step2}
	\min_{\w \geq 0} \ssum_{i=1}^n \exp \biggr[-z_{k,i} \ssum_{q=1}^Q  w_q \tree_q(\x_i) \biggr].
\end{align}
We apply Adaboost to solve above problem.
In each boosting iteration, a decision tree as well as its weighting coefficient are learned.
Every node of a binary decision tree is a decision stump.
Training a stump is to find a feature dimension and threshold that minimizes the weighted classification error.
From this point, we are doing feature selection and hash function learning at the same time.
We can easily make use of efficient decision tree learning techniques available in the literature,
which are able to significantly speed up the training.
Here we summarize some techniques that are included in our implementation:
(i) We have used the highly efficient stump implementation proposed in the recent work of \cite{appelquickly}, which is around 10 times faster than conventional stump implementation.
(ii) Feature quantization can significantly speed up tree training without performance loss in practice, and also largely reduce the memory consuming. As in \cite{appelquickly}, we linearly quantize feature values into 256 bins.
(iii) We apply the weight-trimming technique described in \cite{friedman2000additive, appelquickly}.
In each boosting iteration, the smallest $10\%$ weightings are trimmed (set to 0).
(iv) We apply the LazyBoost technique:
only a random subset of feature dimensions are evaluated for tree node splitting.

Finally, we summarize our hashing method (\fasth) in Algorithm \ref{alg:main}.
In contrast with TSH, we alternate Step-1 and Step-2 iteratively.
For each bit, the binary code is updated by applying the learned hash function.
Hence, the learned hash function is able to make a feedback for binary code inference of next bit,
which may lead to better performance.

\begin{table}[t]
\caption{Comparison of KSH and our \fasth. KSH results
with different number of support vectors.
Both of our \fasth and FastHash-Full outperform KSH by a large margin in terms of training time, binary encoding time (Test time) and retrieval precision.}
\centering
\resizebox{.9\linewidth}{!}
  {
  \begin{tabular}{ l | l c | l c c}
  \hline\hline
Method  &\#Train  &\#Support Vector &Train time &Test time  &Precision\\
\hline
\multicolumn{6}{  c }{CIFAR10 (features:11200)} \\ \hdashline
KSH &5000 &300  &1082 &22 &0.480\\
KSH &5000 &1000 &3481 &57 &0.553\\
KSH &5000 &3000 &52747  &145  &0.590\\
\best FastH &5000 &N/A  &331  &21 &\best 0.634\\
\best FastH-Full  &50000  &N/A  &1794 &21 &\best 0.763\\
\hline
\multicolumn{6}{  c }{IAPRTC12 (features:11200)} \\ \hdashline
KSH &5000 &300  &1129 &7  &0.199\\
KSH &5000 &1000 &3447 &21 &0.235\\
KSH &5000 &3000 &51927  &51 &0.273\\
\best FastH &5000 &N/A  &331  &9  &\best 0.285\\
\best FastH-Full  &17665  &N/A  &620  &9  &\best 0.371\\
\hline
\multicolumn{6}{  c }{ESPGAME (features:11200)} \\ \hdashline
KSH &5000 &300  &1120 &8  &0.124\\
KSH &5000 &1000 &3358 &22 &0.139\\
KSH &5000 &3000 &52115  &46 &0.163\\
\best FastH &5000 &N/A  & 309 &9  &\best 0.188\\
\best FastH-Full  &18689  &N/A  &663  &9  &\best 0.261\\
\hline
\multicolumn{6}{  c }{MIRFLICKR (features:11200)} \\ \hdashline
KSH &5000 &300  &1036 &5  &0.387\\
KSH &5000 &1000 &3337 &13 &0.407\\
KSH &5000 &3000 &52031  &42 &0.434\\
\best FastH &5000 &N/A  &278  &7  &\best0.555\\
\best FastH-Full  &12500  &N/A  &509  &7  &\best 0.595\\

  \hline \hline
  \end{tabular}
  }
\label{tab:ksh}
\end{table}

\section{Experiments}

We here describe the results of
comprehensive experiments
carried out
on several large image datasets
in order to evaluate the proposed method
in terms of training time, binary encoding time and retrieval performance. 
For decision tree learning in our \fasth, if not specified, the tree depth is set to 4,
and the number of boosting iterations is set to 200.
We compare to a number of recent supervised and unsupervised hashing methods.
The retrieval performance is measured in 3 ways: the precision of top-K ($K=100$)
  retrieved examples (denoted as Precision), mean average precision (MAP) and the area under the Precision-Recall curve (Prec-Recall).
Results are reported on 5 image datasets which cover a wide variety of images.
The dataset CIFAR10\footnote{\url{http://www.cs.toronto.edu/~kriz/cifar.html}
}
  contains $60,000$ images. The datasets IAPRTC12 and ESPGAME \cite{guillaumin2009tagprop}
contain
around $20,000$ images,
and
MIRFLICKR \cite{huiskes2008mir} is a collection of $25,000$ images. SUN397
\cite{xiao2010sun} is a large image dataset which contains more than $100,000$ scene images
form 397 categories.

For the
multi-class datasets: CIFAR10 and SUN397, the ground truth pairwise similarity is defined as multi-class label agreement. For datasets: IAPRTC12, ESPGAME and MIRFLICKR, of which the keyword (tags) annotation are provided in \cite{guillaumin2009tagprop}, two images are treated as semantically similar if they are annotated with at lease 2 identical keywords (or tags).
Following a conventional setting in \cite{KSH, kulis2009learning},
a large portion of the dataset is
allocated as an image database for training and retrieval %
and the rest is put aside
for
testing queries. Specifically, for CIFAR10, IAPRTC12, ESPGAME and MIRFLICKER, the provided splits are used; for SUN397, 8000 images are randomly selected as
test %
queries,
while the remaining 100417 images form
the training set.
If not specified, $64$-bit binary codes are generated using comparing methods for evaluation.

We extract codebook-based features following
the conventional pipeline
from \cite{coates2011importance, kiros12}:
we employ K-SVD for codebook (dictionary) learning with a codebook size of 800, soft-thresholding for patch encoding and spatial pooling of 3 levels, which results 11200-dimensional features. We also
tested increasing
 the codebook size to 1600
 which results in
  22400-dimensional features. %

\begin{table}[t]
\caption{Comparison of TSH and our \fasth for binary code inference  in Step 1.
  The proposed Block GraphCut (Block-GC)
    achieves much lower objective value and also takes less inference time than the spectral method,
             and thus performs much better.
}
\centering
\resizebox{.9\linewidth}{!}
  {
  \begin{tabular}{ l | l c | c c}
  \hline\hline
Step-1 methods  &\#train &Block Size  & Time (s)  &Objective\\
\hline
\multicolumn{5}{  c }{SUN397} \\ \hdashline
Spectral (TSH)  &100417 &N/A  &5281 &0.7524\\
Block-GC-1 (FastH)  &100417 &1  &\best 298  &0.6341\\
Block-GC (FastH)  &100417 &253  &2239 &\best 0.5608\\
\hline
\multicolumn{5}{  c }{CIFAR10} \\ \hdashline
Spectral (TSH)  &50000  &N/A  &1363 &0.4912\\
Block-GC-1 (FastH)  &50000  &1  &\best 158  &0.5338\\
Block-GC (FastH)  &50000  &5000 &788  &\best 0.4158\\
\hline
\multicolumn{5}{  c }{IAPRTC12} \\ \hdashline
Spectral (TSH)  &17665  &N/A  &426  &0.7237\\
Block-GC-1 (FastH)  &17665  &1  &\best 43 &0.7316\\
Block-GC (FastH)  &17665  &316  &70 &\best 0.7095\\
\hline
\multicolumn{5}{  c }{ESPGAME} \\ \hdashline
Spectral (TSH)  &18689  &N/A  &480  &0.7373\\
Block-GC-1 (FastH)  &18689  &1  &\best 45 &0.7527\\
Block-GC (FastH)  &18689  &336  &72 &\best 0.7231\\
\hline
\multicolumn{5}{  c }{MIRFLICKR} \\ \hdashline
Spectral (TSH)  &12500  &N/A  &125  &0.5718\\
Block-GC-1 (FastH)  &12500  &1  &\best 28 &0.5851\\
Block-GC (FastH)  &12500  &295  &40 &\best 0.5449\\
  \hline \hline
  \end{tabular}
  }
\label{tab:tsh_step1}
\end{table}

\begin{figure*}
    \centering

   \includegraphics[width=.28\linewidth]{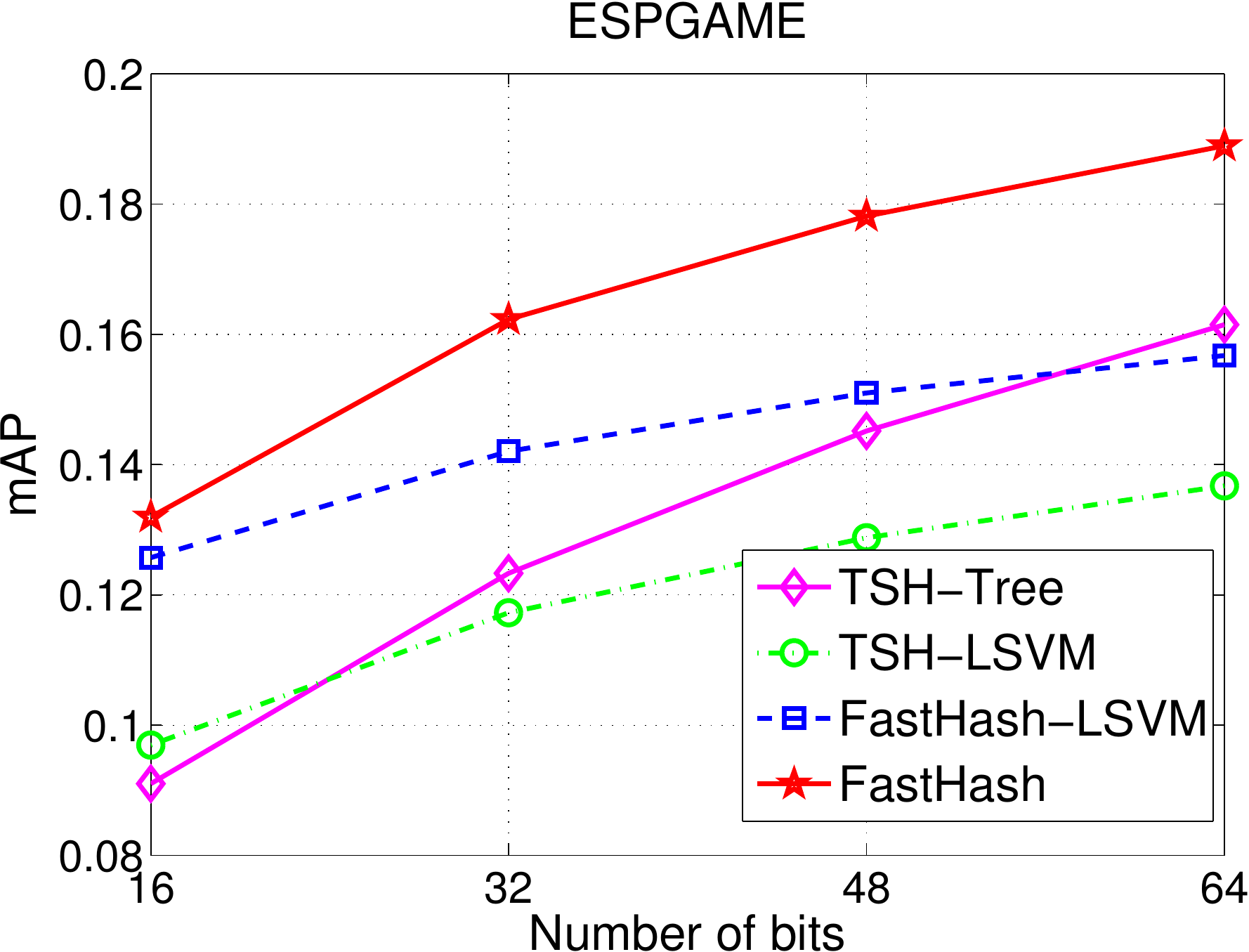}
   \includegraphics[width=.28\linewidth]{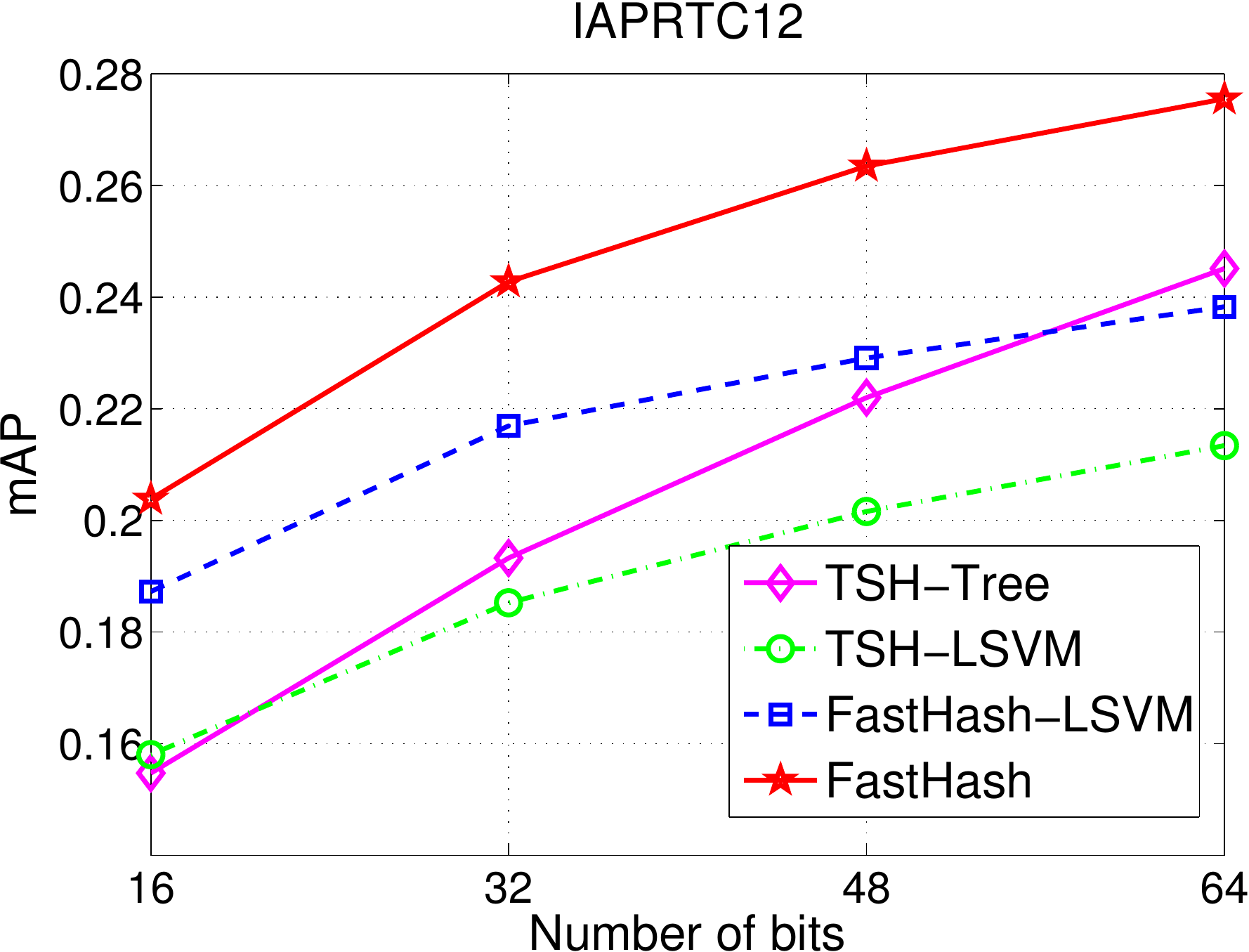}
   \includegraphics[width=.28\linewidth]{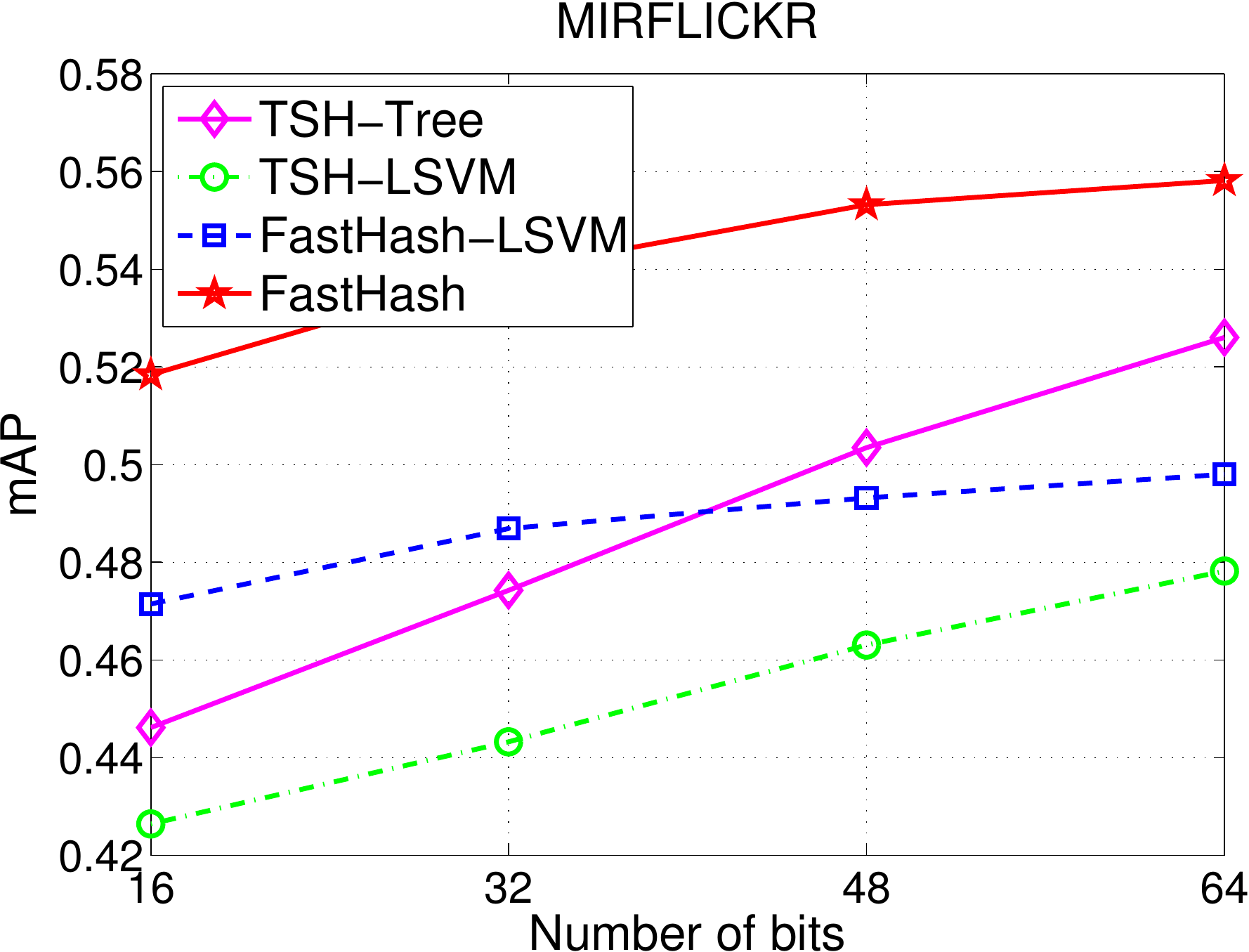}

   \includegraphics[width=.28\linewidth]{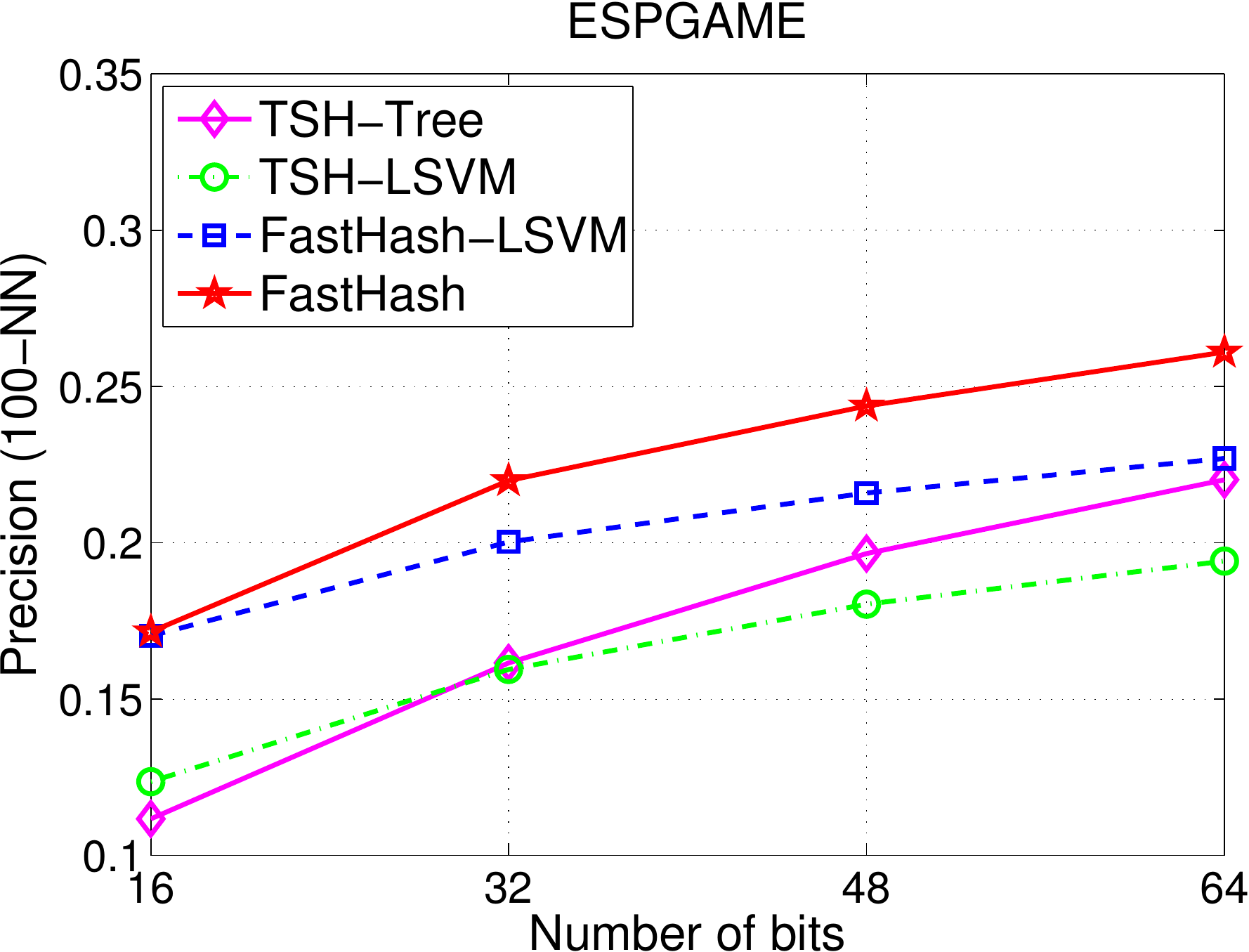}
   \includegraphics[width=.28\linewidth]{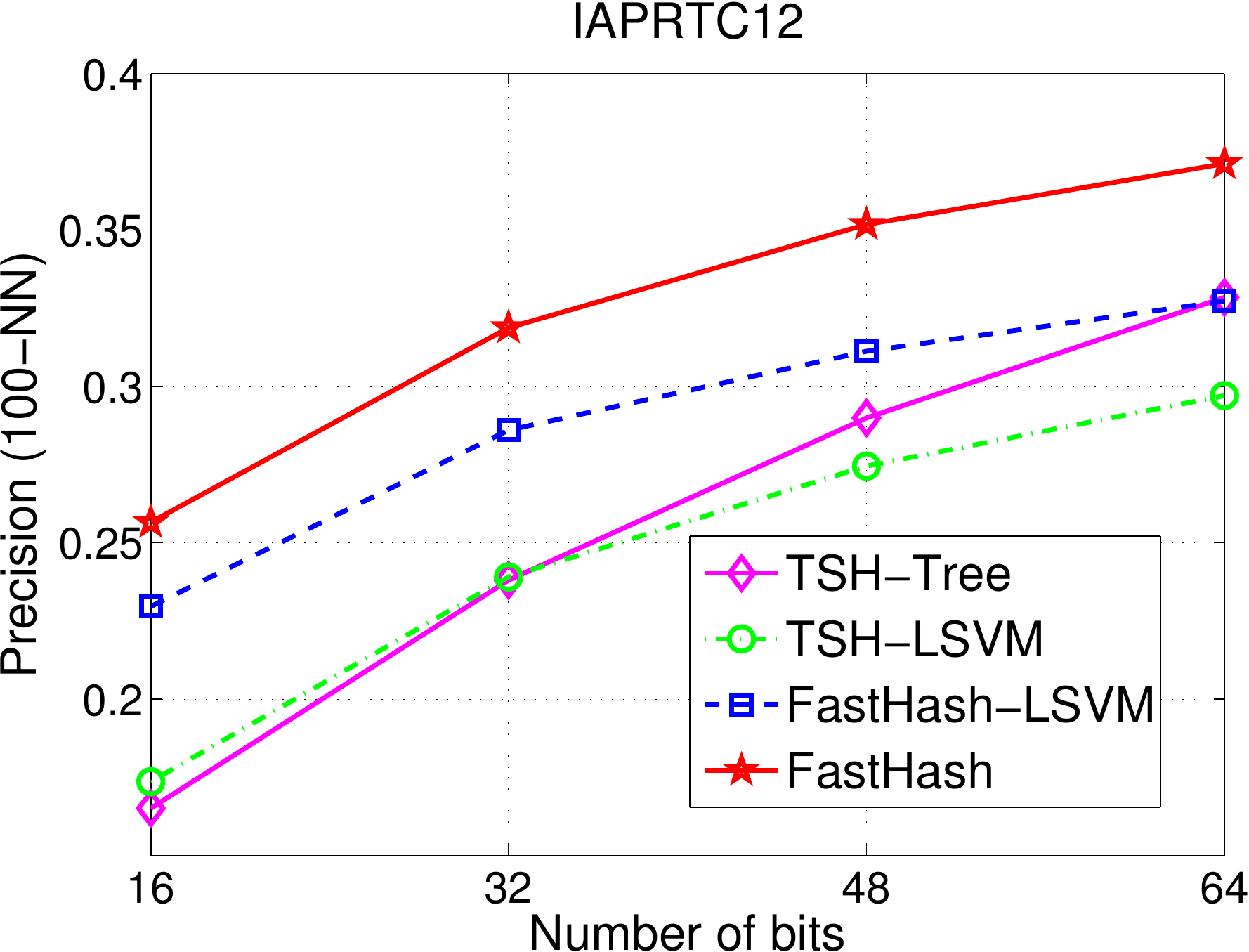}
   \includegraphics[width=.28\linewidth]{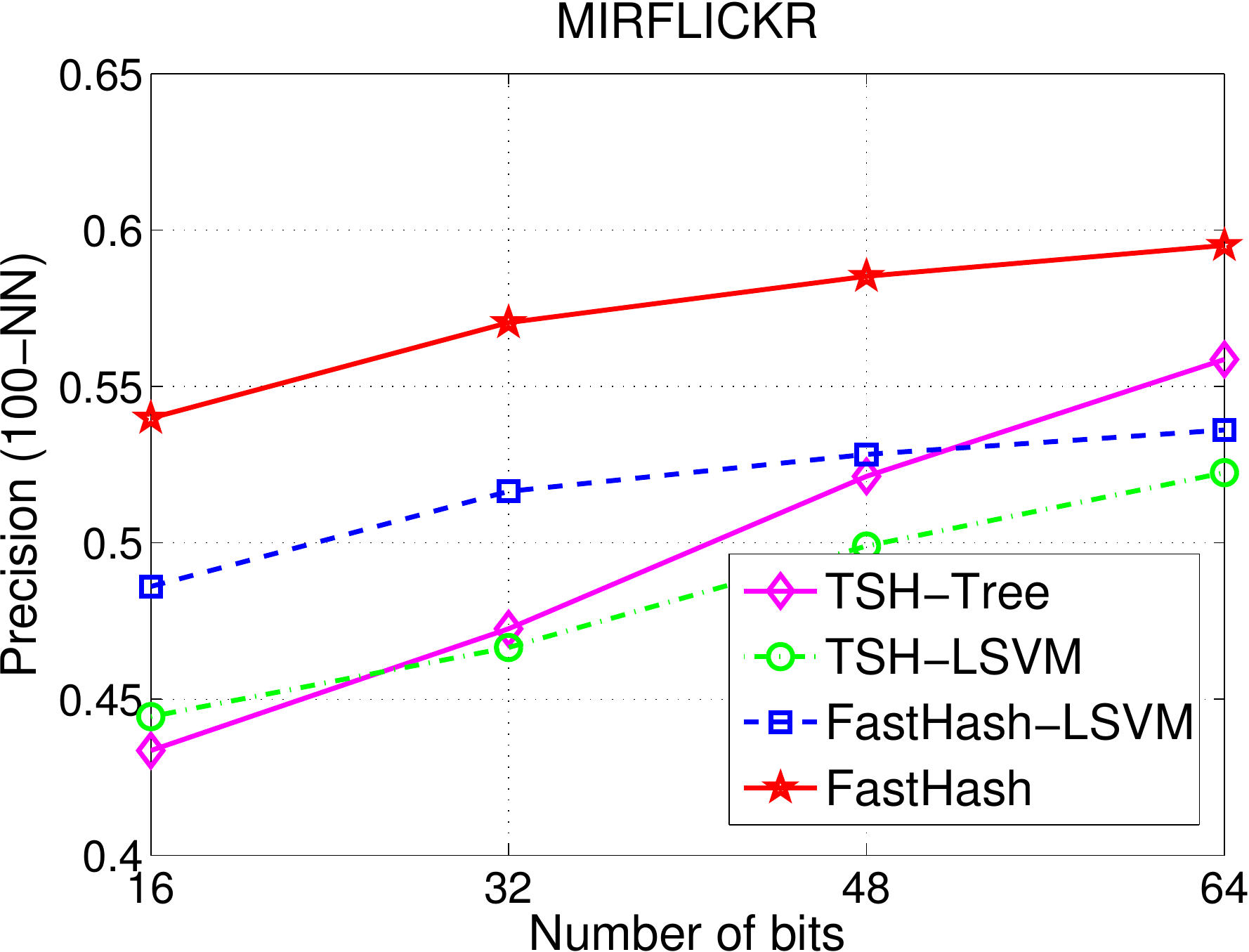}

    \caption{Comparison of
    various
    combinations of hash functions and binary inference methods.
    Note that the proposed \fasth
    uses decision tree as hash functions. The proposed decision tree hash function performs much better than the linear SVM hash function.
Moreover, our \fasth performs much better than TSH when using the same hash function in Step 2.}
    \label{fig:tsh}
\end{figure*}

\begin{figure*}
    \centering

   \includegraphics[width=.28\linewidth]{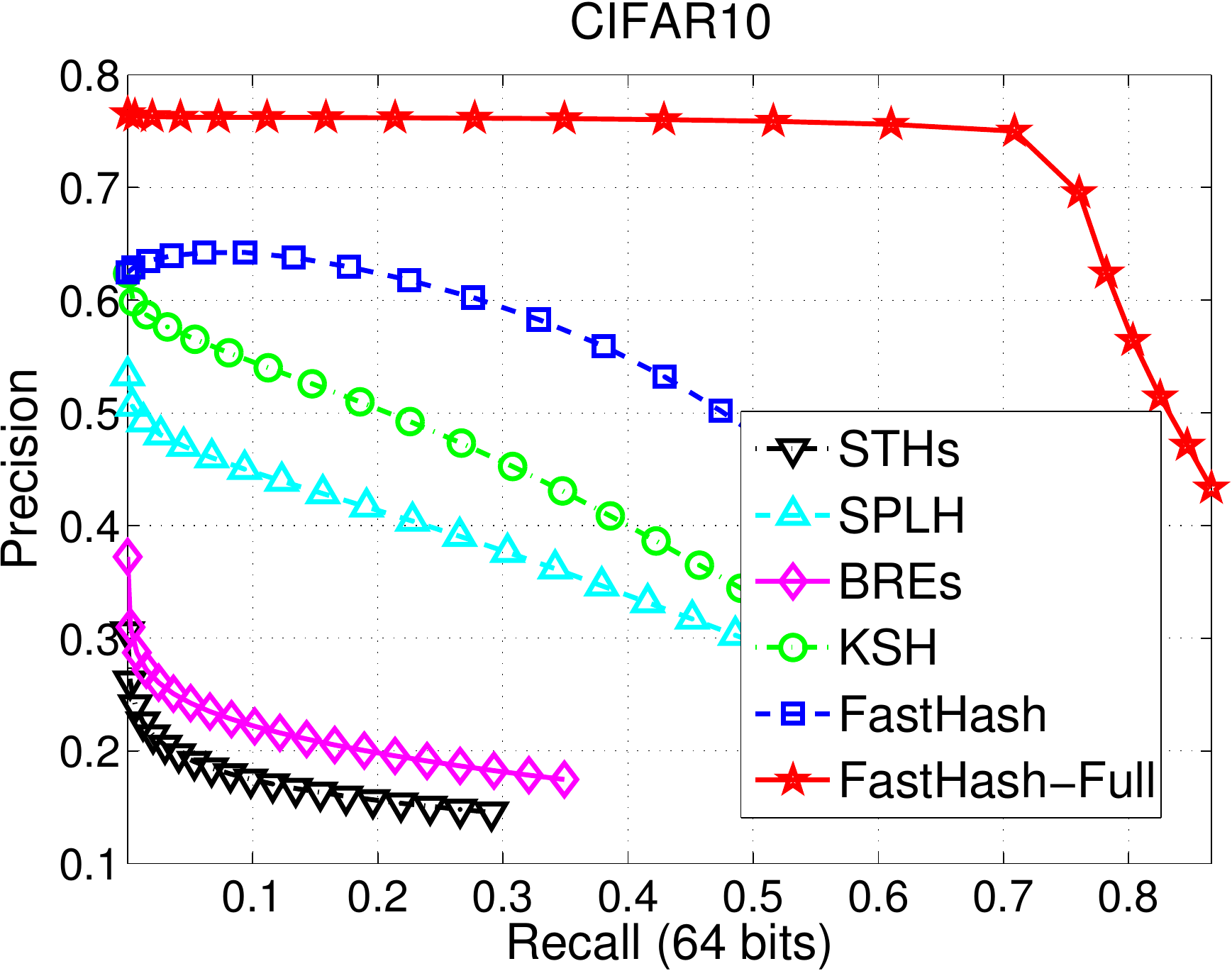}
   \includegraphics[width=.28\linewidth]{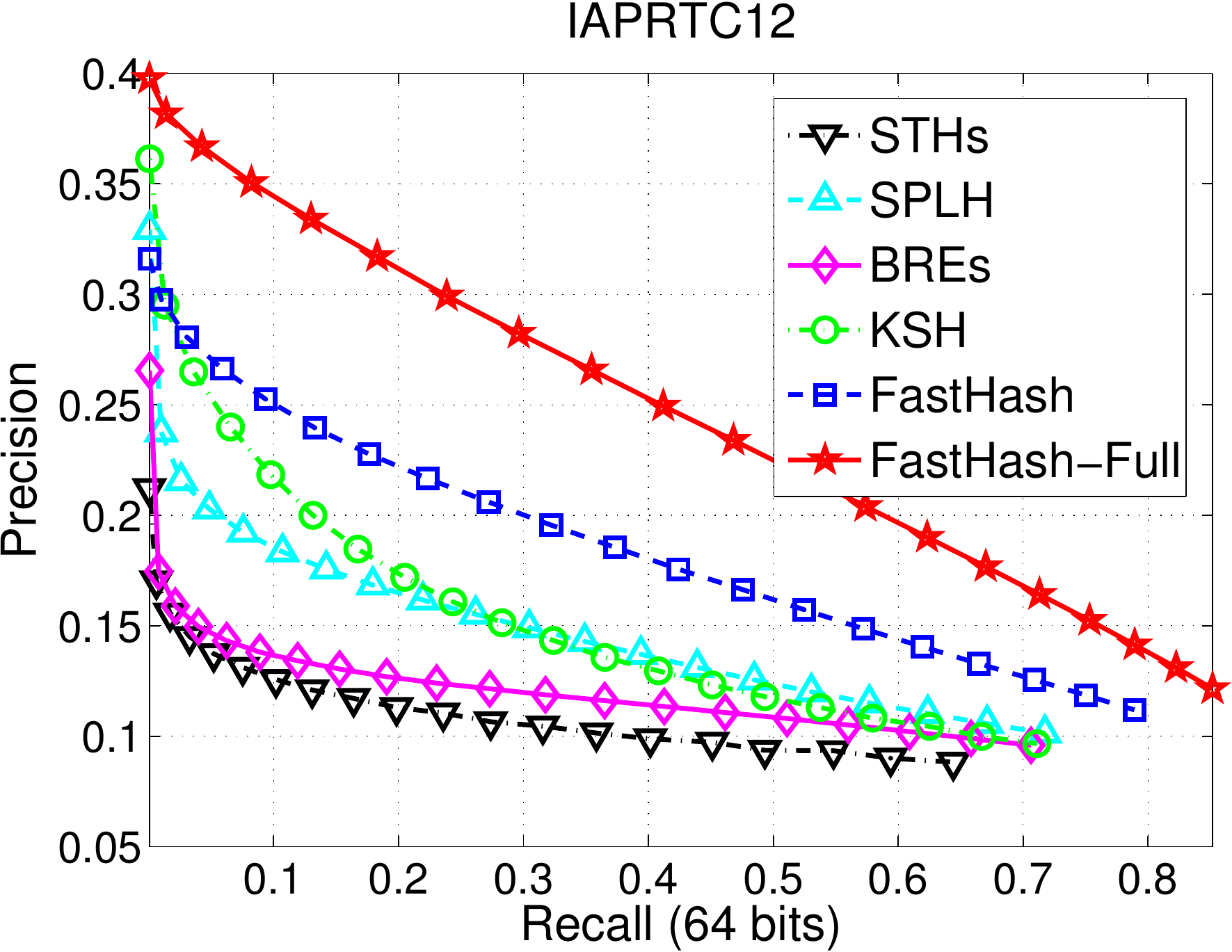}
   \includegraphics[width=.28\linewidth]{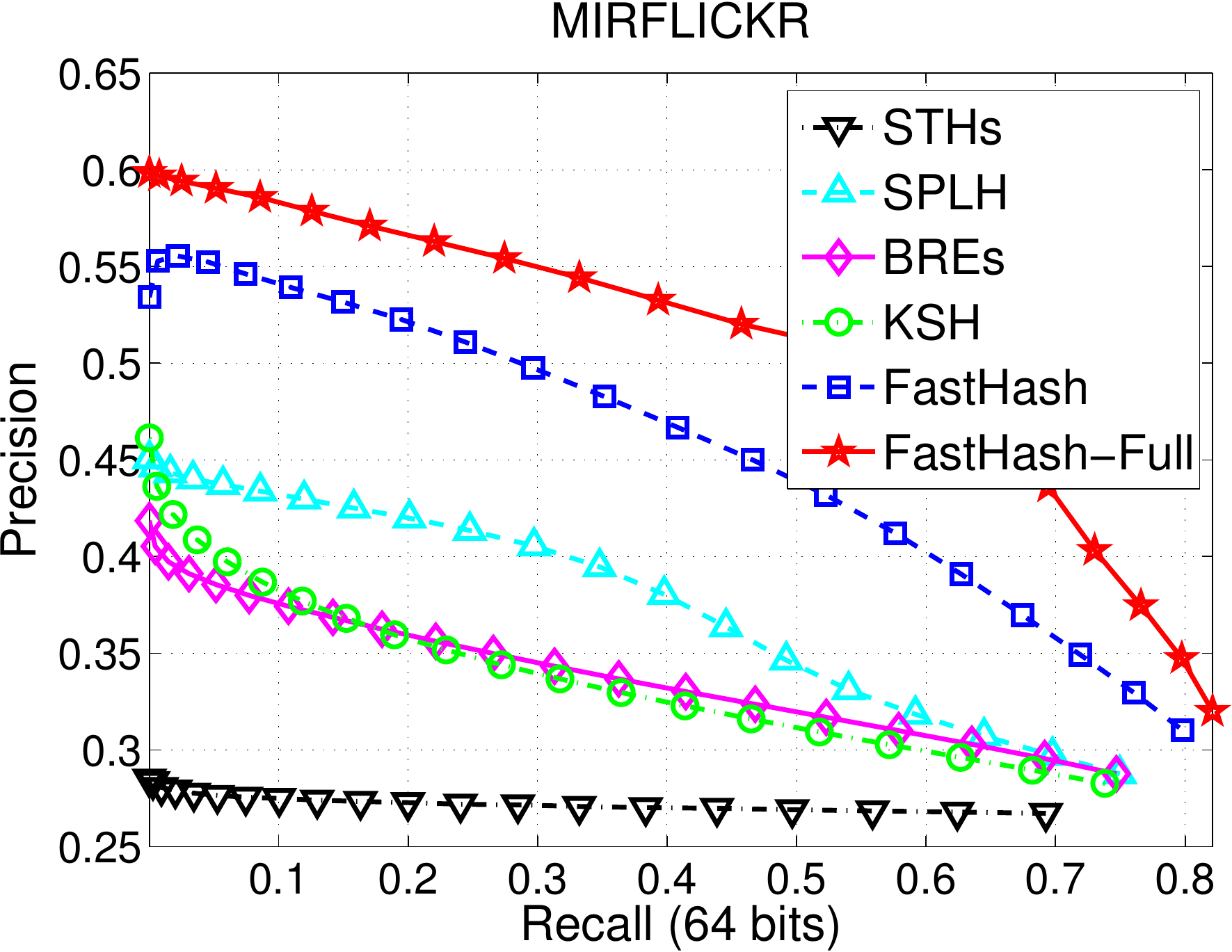}

   \includegraphics[width=.28\linewidth]{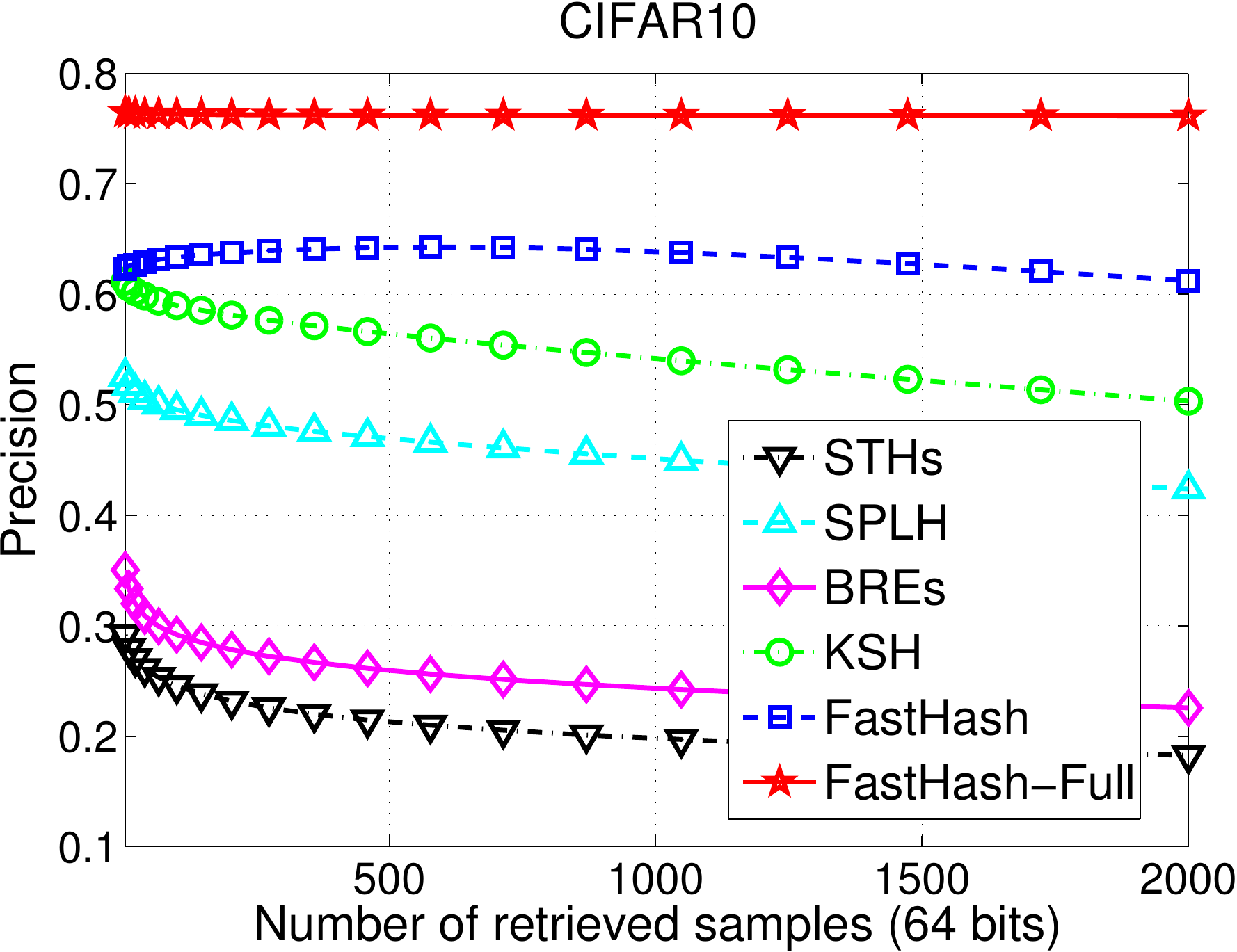}
   \includegraphics[width=.28\linewidth]{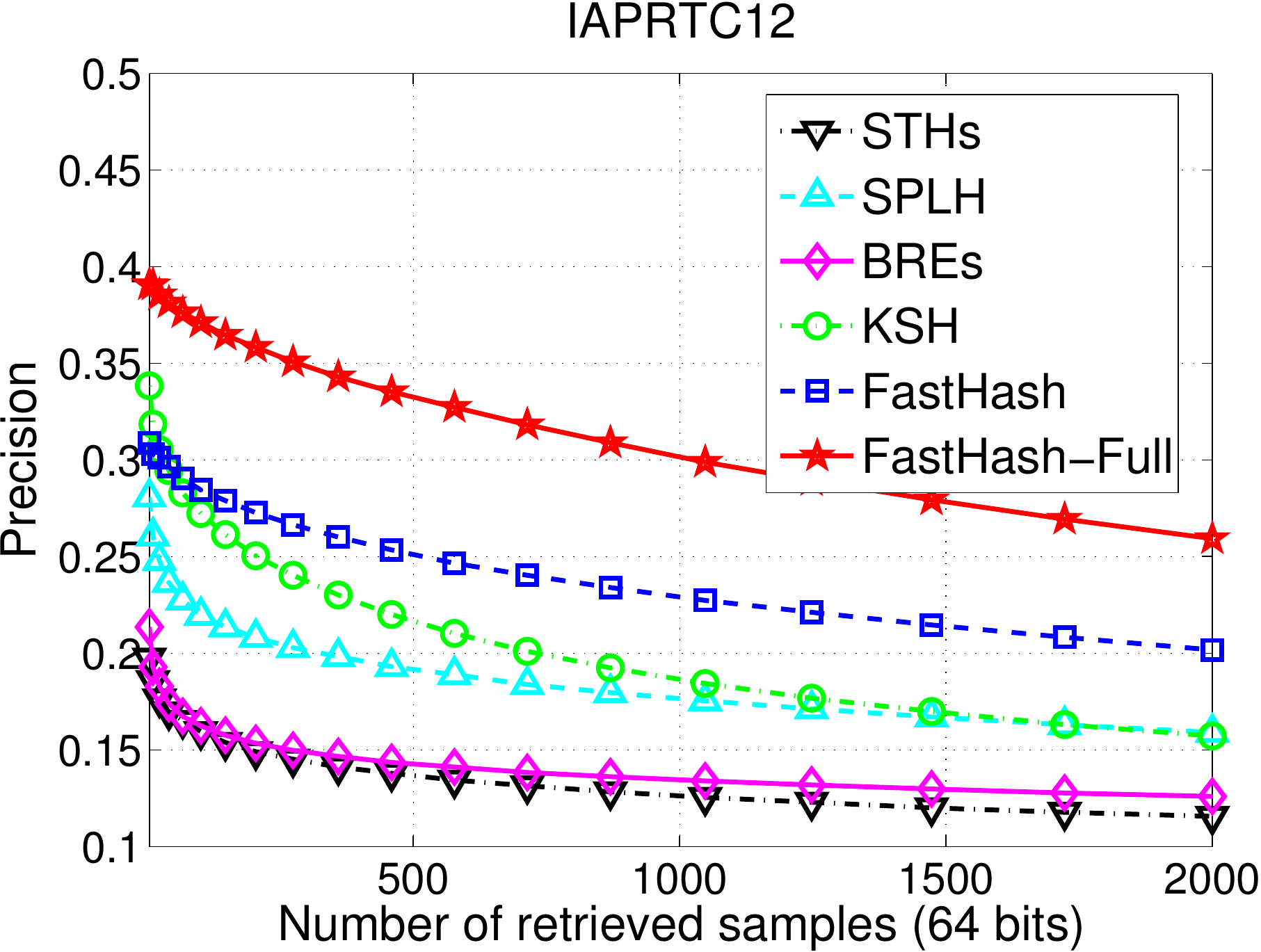}
   \includegraphics[width=.28\linewidth]{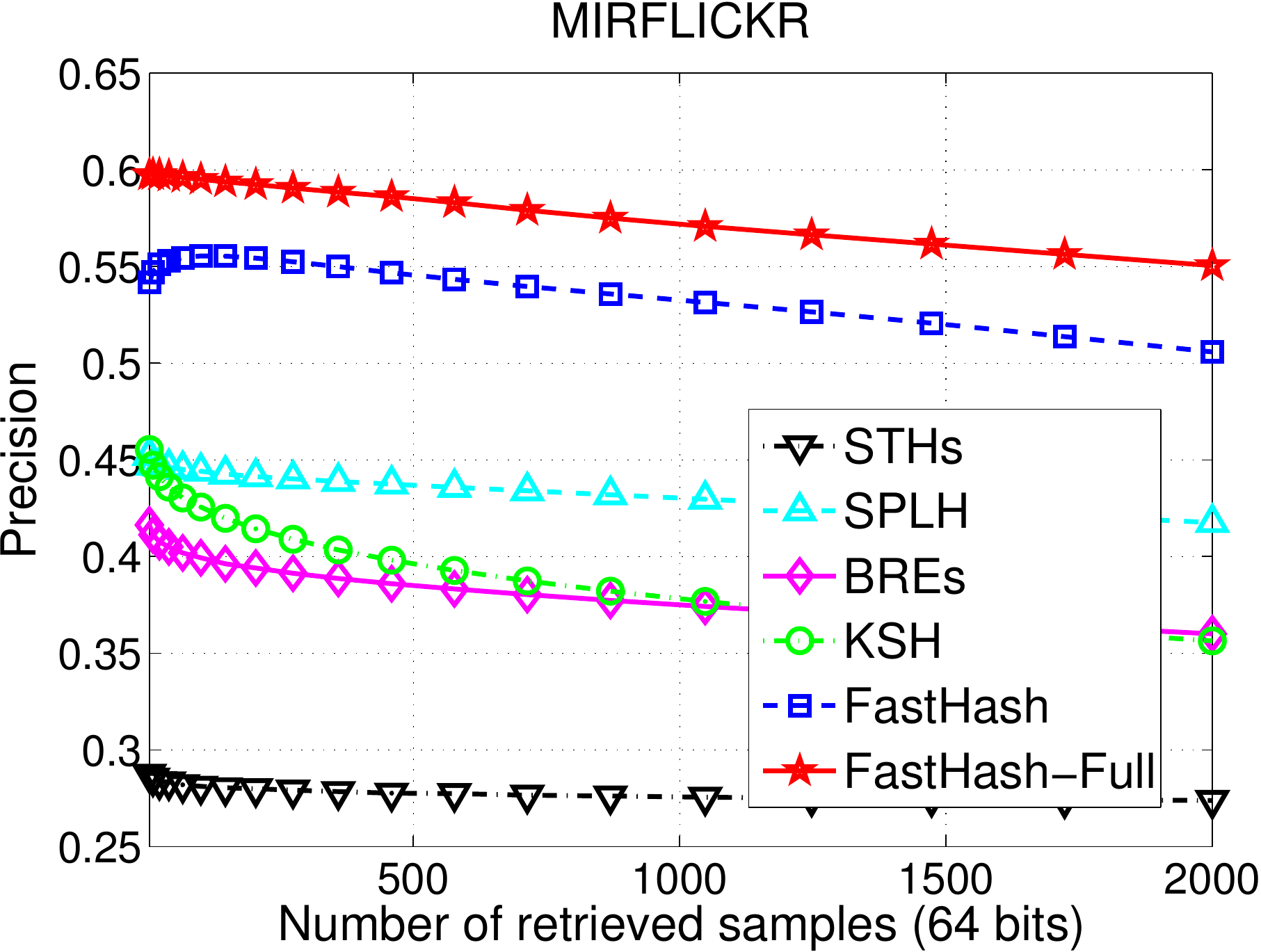}

    \caption{Results on high-dimensional codebook features. The precision and recall curves are given in the first row. The precision curves of the top 2000 retrieved examples are given on the second row.
    Both our \fasth and FastHash-Full outperform their comparators by a large margin.}
    \label{fig:features}
\end{figure*}

\subsection{Comparison with KSH}
KSH \cite{KSH} has been shown to outperform many state-of-the-art comparators.
The fact that our method employs the same loss function as KSH thus motivates further comparison against this key method.
KSH employs a simple kernel technique by predefining a set of support vectors then learning linear weightings for each hash function.
In the works of \cite{KSH, TSH}, KSH is evaluated only on  low dimensional GIST features (512 dimensions) using a small number of support vectors (300).
Here, in contrast, we evaluate KSH on  high-dimensional codebook features,
and vary the number of support vectors from 300 to 3000.
KSH is trained on a sampled set of 5000 examples.
The results of these tests are summarized in Table \ref{tab:ksh},
which shows that increasing the number of support vectors consistently improves the retrieval performance of KSH.
However, even on this small training set, including more support vectors will dramatically increase the training time and binary encoding time of KSH.
We have run our \fasth both on the same sampled training set and the whole training set (labeled as FastHash-Full).
Our \fasth and FastHash-Full outperform KSH by a large margin both in terms of training speed and and retrieval precision. The results also show that the decision tree hash functions in \fasth are much more efficient for testing (binary encoding) than the kernel function in KSH.
Our \fasth is orders of magnitude faster than KSH on training,
and thus much better suited to large training sets and high-dimensional data.
For the low-dimensional GIST feature, our \fasth also performs much better than KSH in retrieval, see Table \ref{tab:features} for details.
The retrieval performance is also plotted in Fig. \ref{fig:features}. 
If not specified,
the number of support vectors for KSH is set to 3000.

\begin{table*}[t]
\caption{Results using two types of features: low-dimensional GIST features and the
  high-dimensional codebook features. Our \fasth and FastHash-Full outperform
  the comparators by a large margin on both feature types.
  In terms of training time, our \fasth is also much faster than others on the high-dimensional codebook features.}
\centering
\resizebox{0.85\linewidth}{!}
  {
  \begin{tabular}{ c l || l l c c c | l l c c c}
  \hline \hline
    & & \multicolumn{5}{ c | }{GIST feature (320 / 512 dimensions)}
& \multicolumn{5}{ c }{Codebook feature (11200 dimensions)}
  \\ \hline
Method &\#Train &Train time &Test time  &Precision  &MAP  &Prec-Recall  &Train time (s) &Test time (s)
 &Precision  &MAP  &Prec-Recall
\\
\hline
\multicolumn{12}{ c }{CIFAR10} \\  \hdashline
KSH &5000 &52173  &8  &0.453  &0.350  &0.164  &52747  &145  &0.590  &0.464  &0.261  \\
BREs  &5000 &481  &1  &0.262  &0.198  &0.082  &18343  &8  &0.292  &0.216  &0.089  \\
SPLH  &5000 &102  &1  &0.368  &0.291  &0.138  &9858 &4  &0.496  &0.396  &0.219  \\
STHs  &5000 &380  &1  &0.197  &0.151  &0.051  &6878 &4  &0.246  &0.175  &0.058  \\
\best FastH &5000 &304  &21 &\best 0.517  &\best 0.462  &\best 0.243  &331  &21 &\best 0.634  &\best 0.575  &\best 0.358  \\
\best FastH-Full  &50000  &1681 &21 &\best 0.649  &\best 0.653  &\best 0.450  &1794 &21 &\best 0.763  &\best 0.775  &\best 0.605  \\
\hline
\multicolumn{12}{ c }{IAPRTC12} \\  \hdashline
KSH &5000 &51864  &5  &0.182  &0.126  &0.083  &51927  &51 &0.273  &0.169  &0.123  \\
BREs  &5000 &6052 &1  &0.138  &0.109  &0.074  &6779 &3  &0.163  &0.124  &0.097  \\
SPLH  &5000 &154  &1  &0.160  &0.124  &0.084  &10261  &2  &0.220  &0.157  &0.119  \\
STHs  &5000 &628  &1  &0.099  &0.092  &0.062  &10108  &2  &0.160  &0.114  &0.076  \\
\best FastH &5000 &286  &9  &\best 0.232  &\best 0.168  &\best 0.117  &331  &9  &\best 0.285  &\best 0.202  &\best 0.146  \\
\best FastH-Full  &17665  &590  &9  &\best 0.316  &\best 0.240  &\best 0.178  &620  &9  &\best 0.371  &\best 0.276  &\best 0.210  \\
\hline
\multicolumn{12}{ c }{ESPGAME} \\  \hdashline
KSH &5000 &52061  &5  &0.118  &0.077  &0.054  &52115  &46 &0.163  &0.100  &0.072  \\
BREs  &5000 &714  &1  &0.095  &0.070  &0.050  &16628  &3  &0.111  &0.076  &0.059  \\
SPLH  &5000 &185  &1  &0.160  &0.124  &0.084  &11740  &2  &0.148  &0.104  &0.074  \\
STHs  &5000 &616  &1  &0.099  &0.092  &0.062  &11045  &2  &0.087  &0.064  &0.042  \\
\best FastH &5000 &289  &9  &\best 0.157  &\best 0.106  &\best 0.070  &309  &9  &\best 0.188  &\best 0.125  &\best 0.081  \\
\best FastH-Full  &18689  &448  &9  &\best 0.228  &\best 0.169  &\best 0.109  &663  &9  &\best 0.261  &\best 0.189  &\best 0.126  \\
\hline
\multicolumn{12}{ c }{MIRFLICKR} \\  \hdashline
KSH &5000 &51983  &3  &0.379  &0.321  &0.234  &52031  &42 &0.434  &0.350  &0.254  \\
BREs  &5000 &1161 &1  &0.347  &0.310  &0.224  &13671  &2  &0.399  &0.345  &0.250  \\
SPLH  &5000 &166  &1  &0.379  &0.337  &0.241  &9824 &2  &0.444  &0.391  &0.277  \\
STHs  &5000 &613  &1  &0.268  &0.261  &0.172  &10254  &2  &0.281  &0.272  &0.174  \\
\best FastH &5000 &307  &7  &\best 0.477  &\best 0.429  &\best 0.299  &338  &7  &\best 0.555  &\best 0.487  & \best 0.344 \\
\best FastH-Full  &12500  &451  &7  &\best 0.525  &\best 0.507  &\best 0.345  & 509 &7  &\best 0.595  &\best 0.558  &\best 0.420  \\
  \hline \hline
  \end{tabular}
  }
\label{tab:features}
\end{table*}

\subsection{Comparison with TSH}
The proposed \fasth employs a similar two-step approach to that of TSH \cite{TSH}.
We first compare binary code inference in Step 1: the proposed Block GraphCut (Block-GC) and the spectral method in TSH.
The iteration number of Block-GC is set to 2.
The results of testing are summarized in Table \ref{tab:tsh_step1}.
We construct blocks using Algorithm \ref{alg:block}.
The average block size is reported in the table.
We also evaluate a special case where the block size is set to 1 for Block-CG (labeled as Block-CG-1),
in which case
Block-GC is reduced to the ICM \cite{besag1986statistical, UGM} method.
It shows that when the training set gets larger, the spectral method becomes slow.
The objective value shown in the table is divided by the number of defined pairwise relations.
The proposed Block-GC achieves much lower objective values and takes less inference time, and hence outperforms the spectral method.
The inference time for Block-CG  increases only linearly with the training set size.
We now provide results comparing different combinations of
hash functions (Step 2) and binary code inference methods (Step 1).
We evaluate the linear SVM and the proposed decision tree hash functions with different binary code inference methods (Spectral method in TSH and Block-GC in \fasth).
The 11200-dimensional codebook features are used here.
The retrieval performance is shown in Fig.~\ref{fig:tsh} by varying the number of bits.
As expected, the proposed decision tree hash function performs much better than linear SVM hash function.
It also shows that our \fasth performs much better than TSH when using the same type of hash function for Step 2
(decision tree or linear SVM hash function), which indicates that the proposed Block-GC method for binary code inference and the stage-wise learning strategy is able to generate high quality binary codes.
We also can train RBF-kernel SVM as hashing function in Step 2,
however, as the case here, when applied on large training set and high-dimensional data, the training of RBF SVM almost become intractable. Even using the stochastic kernel SVM (BSGD) \cite{wang2012breaking} with a support vector budget, the training and testing cost are still very expensive. 

\subsection{Comparison on different features}
We compare hashing methods on the the low-dimensional (320 or 512) GIST feature and the
high-dimensional (11200) codebook feature.
We extract GIST features of 320 dimensions for CIFAR10 which contains low resolution
images, and 520 dimensions for other datasets.
Several state-of-the-art supervised methods are included in this comparison:
KSH \cite{KSH}, Supervised Self-Taught Hashing (STHs) \cite{zhang2010self}, and
Semi-supervised Hashing (SPLH)
  \cite{wang2010semi}.
The result is presented in Table \ref{tab:features}.
The codebook features consistently show better result than GIST features.
Comparing methods are trained on a sampled training set (5000 examples).
Results show that comparing methods can be efficiently trained on the GIST features. However, when applied on high dimensional features, even on a small training set (5000), their training time dramatically increase.
Large matrix multiplication and solving eigenvalue problem on a large matrix may account for the expensive computation in these comparing methods.
It would be very difficult to train these methods on the whole training set.
The training time of KSH mainly depends on the number of support vectors (3000 is used here).
We run our \fasth on the same sampled training set (5000 examples) and the whole training set (labeled as FastHash-Full).
Results show that \fasth can be efficiently trained on the whole dataset.
\fasth and FastHash-Full outperform others by a large margin both on GIST and codebook features.
The training of \fasth is also orders of magnitudes faster than others on the high-dimensional codebook features.
The retrieval performance on codebook features is plotted in Fig. \ref{fig:features}.

\begin{table}[t]
\caption{Results of methods with dimension reduction. KSH, SPLH and STHs are trained with PCA feature reduction. Our \fasth outperforms others by a large margin on retrieval performance.}
\centering
\resizebox{.9\linewidth}{!}
  {
  \begin{tabular}{ l l | l l c c}
  \hline\hline
Method  &\# Train &Train time &Test time  &Precision  &MAP\\
\hline
\multicolumn{6}{  c }{CIFAR10} \\ \hdashline
PCA+KSH&50000&$-$ &$-$  &$-$  &$-$\\
PCA+SPLH  &50000  &25984  &18 &0.482  &0.388\\
PCA+STHs  &50000  &7980 &18 &0.287  &0.200\\
CCA+ITQ&50000&1055  &7  &0.676  &0.642\\
\best FastH &50000  &1794&21  &\best 0.763  &\best 0.775  \\
\hline
\multicolumn{6}{  c }{IAPRTC12} \\ \hdashline
PCA+KSH &17665  &55031  &11 &0.082  &0.103\\
PCA+SPLH  &17665  &1855 &7  &0.239  &0.169\\
PCA+STHs  &17665  &2463 &7  &0.174  &0.126\\
CCA+ITQ&17665&804 &3  &0.332  &0.198\\
\best FastH &17665  &620  &9  &\best 0.371  &\best 0.276\\
\hline
\multicolumn{6}{  c }{ESPGAME} \\ \hdashline
PCA+KSH &18689  &55714  &11 &0.141  &0.084\\
PCA+SPLH  &18689  &2409 &7  &0.153  &0.103\\
PCA+STHs  &18689  &2777 &7  &0.098  &0.069\\
CCA+ITQ&18689&814 &3  &0.216  &0.131\\
\best FastH &18689  &663  &9  &\best 0.261  &\best 0.189\\
\hline
\multicolumn{6}{  c }{MIRFLICKR} \\ \hdashline
PCA+KSH &12500  &54260  &8  &0.384  &0.313\\
PCA+SPLH  &12500  &1054 &5  &0.445  &0.391\\
PCA+STHs  &12500  &1768 &5  &0.347  &0.301\\
CCA+ITQ &12500&699  &3  &0.519  &0.408\\
\best FastH &12500  &509  &7  &\best 0.595  &\best 0.558\\

  \hline \hline
  \end{tabular}
  }
\label{tab:pca}
\end{table}

\subsection{Comparison with dimension reduction}
A possible way to reduce the training cost on high-dimensional data is to apply dimension reduction.
For comparing methods: KSH, SPLH and STHs, here we reduce the original 11200-dimensional codebook features to 500 dimensions by applying PCA.
We also compare to CCA+ITQ \cite{gong2012iterative} which combines ITQ with the supervised dimensional reduction.
Our \fasth still use the original high-dimensional features.
The result is summarized in Table \ref{tab:pca}.
After dimension reduction, most comparing methods can be trained on the whole training set within 24 hours
(except KSH on CIFAR10).
However it still much slower than our \fasth.
The retrieval performance of most methods get improved with more training data.
Our \fasth still significantly outperforms all others.
The proposed decision tree hash functions in \fasth actually perform feature selection and hash function learning at the same time, which shows much better performance than other hashing method with dimensional reduction.

\subsection{Comparison with unsupervised methods}
We compare to some popular unsupervised hashing methods:
LSH \cite{Gionis1999},
ITQ \cite{gong2012iterative},
Anchor Graph Hashing
(AGH) \cite{liu2011hashingGraphs}, Spherical Hashing (SPHER) \cite{jae2012},
MDSH \cite{MDSH}.
The retrieval performance is shown in Fig.~\ref{fig:unsup}.
Unsupervised methods perform poorly for preserving label based similarity. Our \fasth outperforms others by a large margin.

\begin{table}[t]
\caption{Performance of our \fasth on more features (22400 dimensions) and more bits (1024 bits). It shows that \fasth can be efficiently trained on high-dimensional features with large bit length. The training and binary coding time (Test time) of \fasth is only linearly increased with the bit length.}
\centering
\resizebox{1\linewidth}{!}
  {
  \begin{tabular}{ l l l | l l c c}
    \hline \hline
Bits   &\#Train  &Features &Train time  &Test time  &Precision  &MAP  \\
\hline
\multicolumn{7}{ c}{CIFAR10}\\ \hdashline
64  &50000  &11200  &1794 &21 &0.763  &0.775  \\
256 &50000  &22400  &5588 &71 &0.794  &0.814  \\
1024  &50000  &22400  &22687  &282  &0.803  &0.826  \\
\hline
\multicolumn{7}{ c}{IAPRTC12} \\ \hdashline
64  &17665  &11200  &320  &9  &0.371  &0.276   \\
256 &17665  &22400  &1987 &33 &0.439  &0.314    \\
1024  &17665  &22400  &7432 &134  &0.483  &0.338   \\
\hline
\multicolumn{7}{ c}{ESPGAME} \\ \hdashline
64  &18689  &11200  &663  &9  &0.261  &0.189\\
256 &18689  &22400  &1912 &34 &0.329  &0.233\\
1024  &18689  &22400  &7689 &139  &0.373  &0.257\\
\hline
\multicolumn{7}{ c}{MIRFLICKR} \\ \hdashline
64  &12500  &11200  &509  &7  &0.595  &0.558\\
256 &12500  &22400  &1560 &28 &0.612  &0.567\\
1024  &12500  &22400  &6418 &105  &0.628  &0.576\\
\hline\hline
  \end{tabular}
  }
\label{tab:long_bits}
\end{table}

\subsection{More features and more bits}
We increase the codebook size to 1600 for generating higher dimensional features (22400 dimensions) and run up to 1024 bits.
The result is shown in Table \ref{tab:long_bits}.
It shows that our \fasth can be efficiently trained on high-dimensional features with large bit length. The training and binary coding time (Test time) of \fasth is only linearly increased with the bit length. 

\begin{table*}[t]
\caption{Results on the large image dataset SUN397 using
11200-dimensional codebook features.
Our \fasth can be efficiently trained to large bit length (1024 bits) on this large training set. Both of our \fasth and FastH-N outperform other methods by a large margin on retrieval performance.}
\centering
\resizebox{.9\linewidth}{!}
  {
  \begin{tabular}{ l l l | l l c c ||  l l l | l l c c}
    \hline \hline
Method  &\#Train  &Bits &Train time &Test time  &Precision  &MAP  &Method &\#Train  &Bits &Train time &Test time  &Precision  &MAP\\
\hline
\multicolumn{14}{ c}{SUN397} \\ \hdashline
KSH &10000  &64 &57045  &463  &0.034  &0.023  &ITQ  &100417 &1024 &1686 &127  &0.030  &0.021\\
BREs  &10000  &64 &105240 &23 &0.019  &0.013  &SPHER  &100417 &1024 &35954  &121  &0.039  &0.024\\
SPLH  &10000  &64 &27552  &14 &0.022  &0.015  &LSH  &100417 &1024 &99 &99 &0.028  &0.019\\
STHs  &10000  &64 &22914  &14 &0.010  &0.008  &CCA+ITQ  &100417 &1024 &15580  &127  &0.120  &0.081\\
CCA+ITQ &100417 &512  &7484 &66 &0.113  &0.076  &\best FastH  &100417 &1024 &62076  &536  &\best 0.165  &\best 0.163\\
\best FastH &100417 &512  &29624  &302  &0.149  &0.142  &\best FastH-N  &100417 &1024 &71203  &749  &\best 0.177  &\best 0.184\\
\hline\hline
  \end{tabular}
  }
\label{tab:sun}
\end{table*}

\begin{figure*}
    \centering

   \includegraphics[width=.22\linewidth]{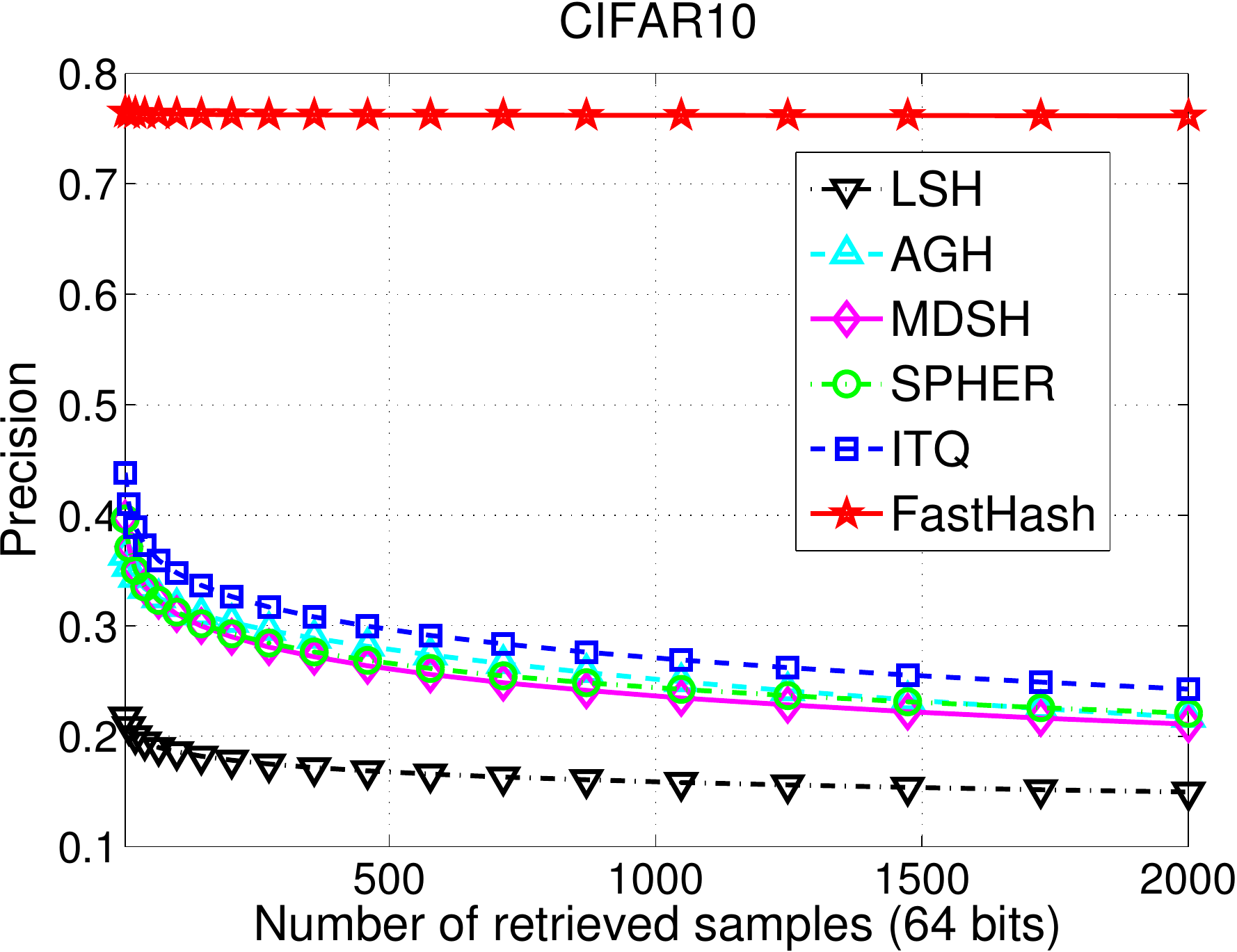}
   \includegraphics[width=.22\linewidth]{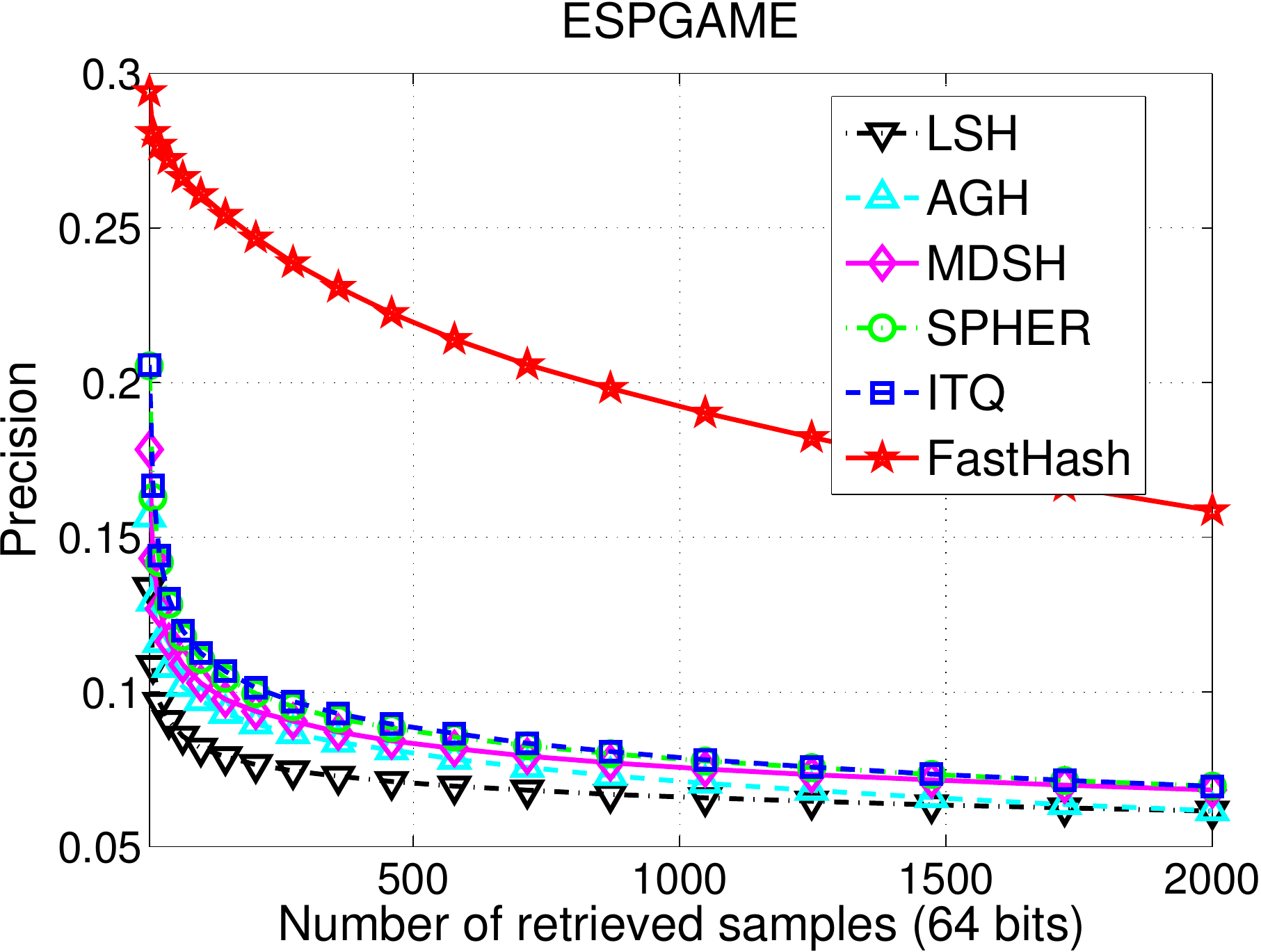}
   \includegraphics[width=.22\linewidth]{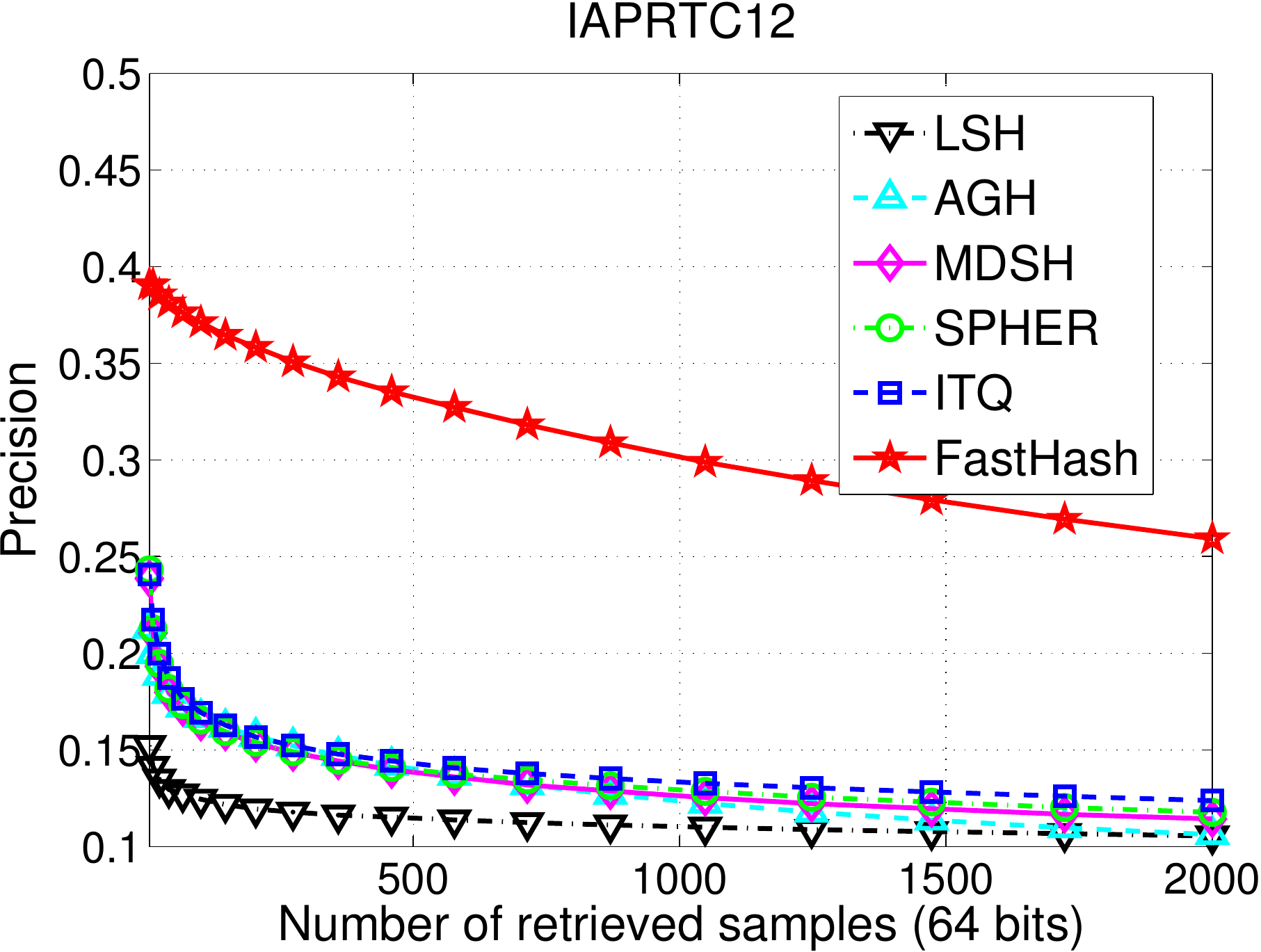}
   \includegraphics[width=.22\linewidth]{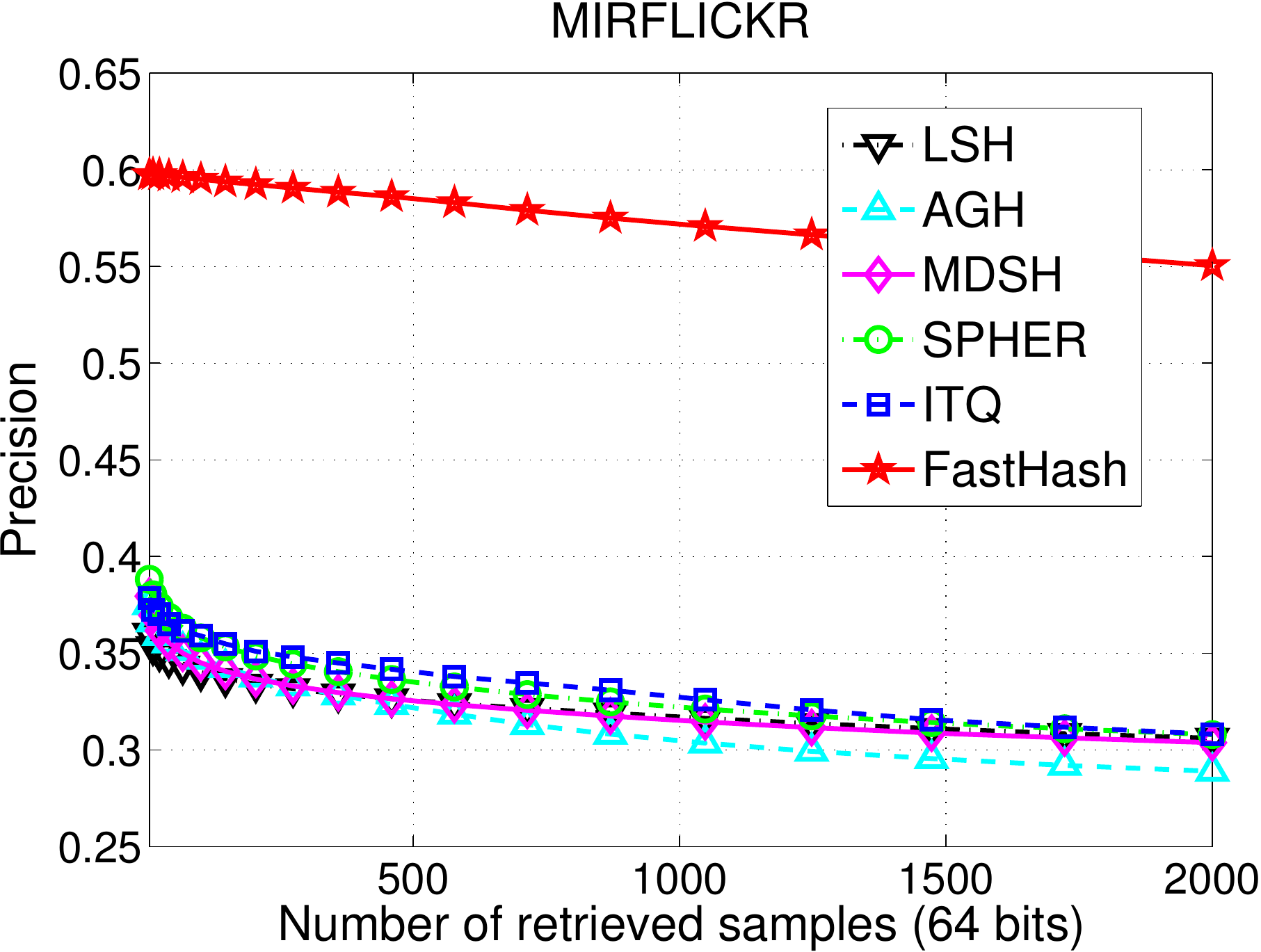}

    \caption{The retrieval precision results of unsupervised methods.  Unsupervised methods perform poorly for preserving label based similarity. Our \fasth outperform others by a large margin.}
    \label{fig:unsup}
\end{figure*}

\subsection{Large dataset: SUN397}
The challenging SUN397 dataset is a collection of more than $100,000$ scene images from 397 categories.
11200-dimensional codebook features are extracted on this dataset.
We compare with a number of supervised and unsupervised methods.
The depth for decision trees is set to 6.
The result is presented in Table \ref{tab:sun}
Supervised methods: KSH, BREs, SPLH and STHs are trained to 64 bits on a subset of 10K examples. 
However, even on this sampled training set and only run to 64 bits, 
the training of these methods are already impractically slow. It would be almost intractable for the whole training set and long bit length. Short length of bits are not able to achieve good performance on this challenging dataset.
In contrast, our method can be efficiently trained to large bit length (1024 bits) on the whole training set 
(more than $100,000$ training examples).
FastH-N is our \fasth using weighted sampling of examples (5000 examples) for tree node splitting.
FastH-N may consume more training time due to less pruning based on minimum node size.
Both of our \fasth and FastH-N outperform other methods by a large margin on retrieval performance.

For memory usage, many of the comparing methods require a large amount of memory for large matrix multiplication. In contrast, the decision tree learning in our method only involves the simple comparison operation on quantized feature data (256 bins), thus \fasth only consumes less than 7GB for training, which shows that our method can be easily applied for large-scale training.

\begin{figure}[t]
    \centering

   \includegraphics[width=.6\linewidth]{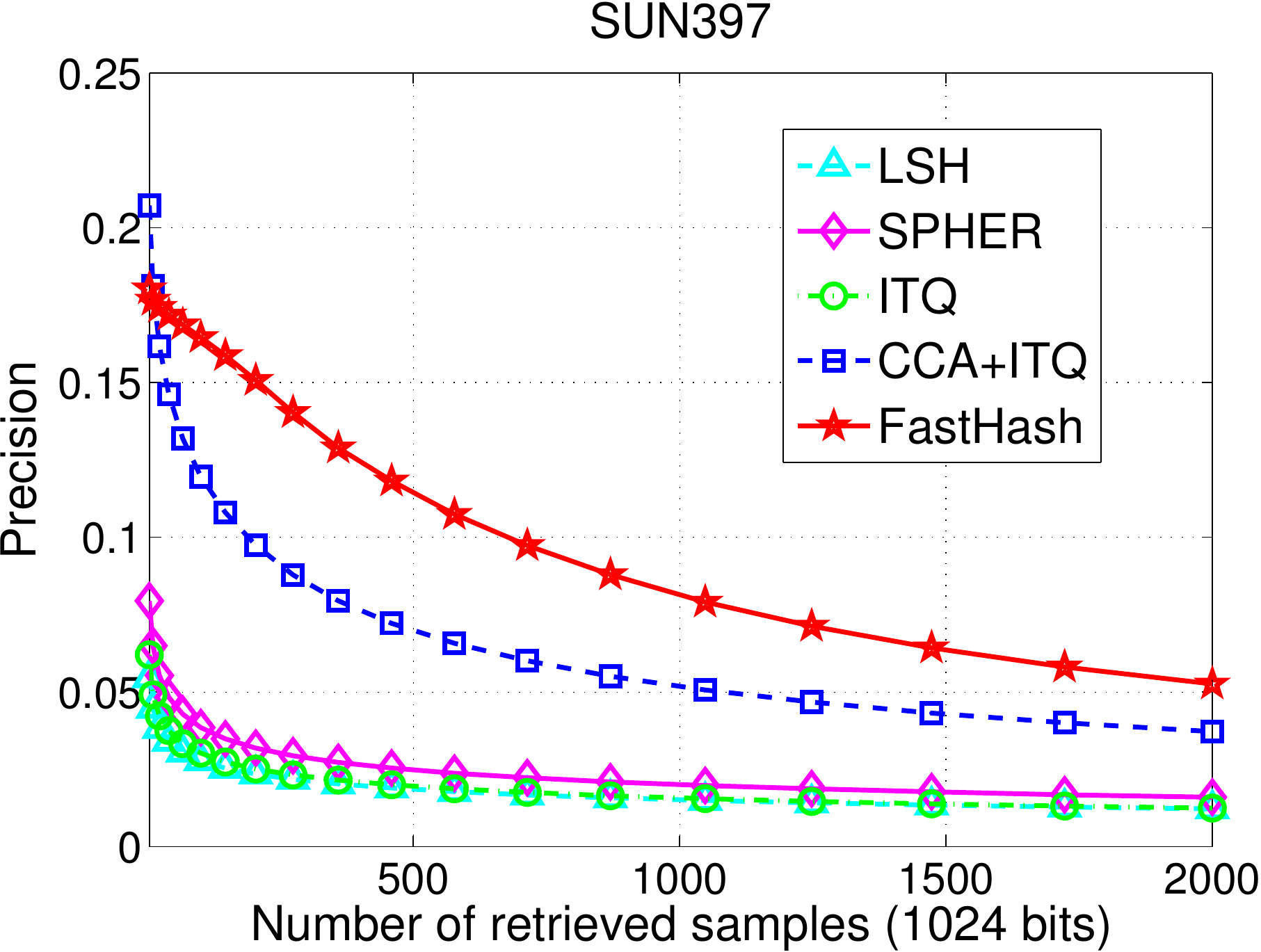}

    \caption{The precision curve of top 2000 retrieved examples on large image dataset SUN397 using 1024 bits.  Here we compare with those methods which can be efficiently trained up to 1024 bits on the whole training set. Our \fasth outperforms others by a large margin. %
    }
    \label{fig:sun}
\end{figure}

\section{Conclusion}
We have proposed an efficient supervised hashing method,
which uses decision tree based hash functions and GraphCut based binary code inference.
Our comprehensive experiments show the advantages of our method on retrieval performance and
fast training for high-dimensional data, which indicates its practical significance on
many potential  applications like large-scale image retrieval.

{\small
\bibliographystyle{ieee}
\bibliography{hash}
}

\end{document}